\documentclass{article}

\usepackage[final]{neurips_2020}

\usepackage[utf8]{inputenc} 
\usepackage[T1]{fontenc}    
\usepackage{hyperref}       
\usepackage{url}            
\usepackage{booktabs}       
\usepackage{amsfonts}       
\usepackage{nicefrac}       
\usepackage{microtype}

\usepackage{authblk}
\usepackage{subfigure}
\usepackage{amsthm}
\usepackage[reqno]{amsmath}
\usepackage{mdframed}
\usepackage{wrapfig}
\usepackage{adjustbox}
\theoremstyle{definition}
\newmdtheoremenv{ex}{Example}

\newtheorem{theo}{Theorem}
\newtheorem{prop}{Proposition}

\newtheorem{lem}{Lemma}
\newtheorem{coro}{Corollary}
\newtheorem{defi}{Definition}

\allowdisplaybreaks[4]
\usepackage[whole]{bxcjkjatype}
\usepackage{multirow}
\usepackage{enumerate}
\usepackage{amsmath,amssymb}
\newcommand{\diff}{\mathrm{d}}
\newcommand{\bs}{\boldsymbol}
\newcommand{\argmin}{\mathop{\arg\min}}

\newcommand{\first}[1]{\textbf{\color{black}{#1}}}
\newcommand{\second}[1]{\underline{#1}}
\def\qed{\hfill $\Box$}

\usepackage{bbm}

\title{{\Large Extrapolation Towards Imaginary $0$-Nearest Neighbour \\ and Its Improved Convergence Rate}}

\author[1,3]{Akifumi Okuno}
\author[2,3]{Hidetoshi Shimodaira}
\affil[1]{School of Statistical Thinking, The Institute of Statistical Mathematics}
\affil[2]{Graduate School of Informatics, Kyoto University}
\affil[3]{RIKEN Center for Advanced Intelligence Project}

\begin{document}

\maketitle

\begin{abstract}
$k$-nearest neighbour ($k$-NN) is one of the simplest and most widely-used methods for supervised classification, that predicts a query's label by taking weighted ratio of observed labels of $k$ objects nearest to the query.
The weights and the parameter $k \in \mathbb{N}$ regulate its bias-variance trade-off, and the trade-off implicitly affects the convergence rate of the excess risk for the $k$-NN classifier; several existing studies considered selecting optimal $k$ and weights to obtain faster convergence rate. 
Whereas $k$-NN with non-negative weights has been developed widely, it was also proved that negative weights are essential for eradicating the bias terms and attaining optimal convergence rate. 
In this paper, we propose a novel \emph{multiscale $k$-NN~(MS-$k$-NN)}, that extrapolates unweighted $k$-NN estimators from several $k \ge 1$ values to $k=0$, thus giving an imaginary 0-NN estimator. 
Our method implicitly computes optimal real-valued weights that are adaptive to the query and its neighbour points.
We theoretically prove that the MS-$k$-NN attains the improved rate, which coincides with the existing optimal rate under some conditions. 
\end{abstract}

\section{Introduction}

Supervised classification has been a fundamental problem in machine learning and statistics over the years. 
It is widely used in a number of applications, such as 
music-genre categorization~\citep{li2003comparative}, 
medical diagnosis~\citep{soni2011predictive}, 
speaker recognition~\citep{ge2017neural} and so forth. 
Moreover, vast amounts of data have become readily available for anyone, along with the development of information technology; 
potential demands for better classification are still growing.

\begin{wraptable}[11]{r}[0em]{60mm}
\vspace{-1.5em}
\centering
\caption{Convergence Rates}
\vspace{0.5em}
\scalebox{0.9}{
\begin{tabular}{ll}
\toprule[0.2ex]
    Nadaraya-Watson & $n^{-4/(4+d)}$ \\
    Local polynomial$^{\dagger}$ & $n^{-2\beta/(2\beta+d)}$ \\
    $k$-NN (unweighted) & $n^{-4/(4+d)}$ \\
    $k$-NN (with weights $\ge 0$) & $n^{-4/(4+d)}$ \\
    $k$-NN (with weights $\in \mathbb{R}$) & $n^{-2\beta/(2\beta+d)}$ \\
    \textbf{Multiscale $k$-NN} & $n^{-2\beta/(2\beta+d)}$ \\
    \hline
\end{tabular}
}
\\ {\footnotesize $^{\dagger}$uniform bound; others are non-uniform.} 
\\ {\footnotesize $\alpha=1,\beta=2u,u \in \mathbb{N}$, $\gamma=2$.}
\label{table:rates}
\end{wraptable}

One of the simplest and most widely-used methods for supervised classification is $k$-nearest neighbour~($k$-NN; \citet{fix1951nonparametric}), where the \emph{estimator} predicts a query's label probability by taking the weighted ratio of observed labels of $k$ objects nearest to the query, and the corresponding \emph{classifier} specifies the class of objects via the predicted label probabilities. 
$k$-NN has strengths in its simplicity and flexibility over and above its statistical consistency~(as $k=k_n \to \infty,k_n/n \to 0,n \to \infty$), proved by \citet{fix1951nonparametric}, \citet{cover1967nearest} and \citet{stone1977consistent}. 
However, such a simple $k$-NN heavily depends on the selection of parameters, i.e., the weights and $k$ therein; inexhaustible discussions on parameter selection have been developed for long decades~\citep{devroye1996probabilistic,boucheron2005theory,audibert2007learning,samworth2012optimal,chaudhuri2014rates,anava2016nearest,cannings2017local,balsubramani2019adaptive}.

\begin{wrapfigure}[17]{r}[0em]{60mm}
\centering
\includegraphics[scale=0.44]{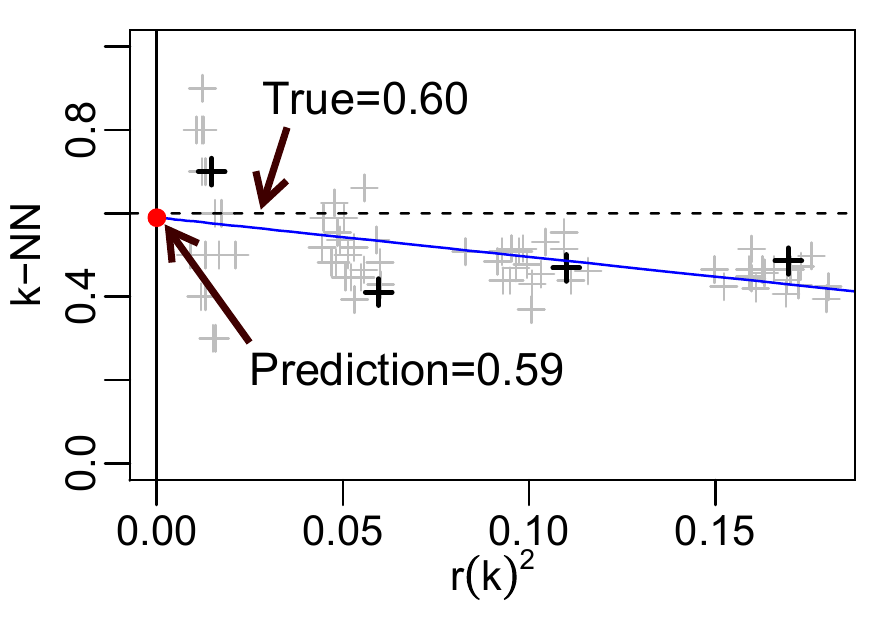}
\vspace{-0.5em}
\caption{For a fixed query $X_* \in \mathbb{R}^{5}$, 
unweighted $k$-NN estimators using synthetic data for 4 different $k$ values are plotted as grey points ($20$ times); bias-variance trade-off is observed. 
In an instance shown as black point
$k$-NN estimators are extrapolated to imaginary $0$-NN by regression (\ref{eq:msknn_regression_function}), via $r(k)^2:=\|X_{(k)}-X_*\|_2^2$ where $X_{(k)}$ is the $k$-th nearest to the query $X_*$.
}
\label{fig:illustration_msknn}
\end{wrapfigure}

A prevailing line of research in the parameter selection focuses on \emph{misclassification error rate} of classifiers as the sample size $n$ grows asymptotically.
They attempt to minimize the \emph{convergence rate} of the \emph{excess risk}, i.e., the difference of error rates between the classifier and Bayes-optimal classifier . The convergence rate depends on the functional form of the conditional expectation $\eta(x):=\mathbb{E}(Y \mid X=x)$ of the binary label $Y \in \{0,1\}$ given its feature vector $X \in \mathcal{X} (\subset \mathbb{R}^d)$. 
Its function class is specified by 
(i) \textbf{$\alpha$-margin condition}, 
(ii) \textbf{$\beta$-H{\"o}lder condition}, 
(iii) \textbf{$\gamma$-neighbour average smoothness}, 
that will be formally described in Definition~\ref{def:alpha_margin}, \ref{def:beta_holder} and \ref{def:gamma_local_homogeneity} later in Section~2. 
Roughly speaking, classification problems with larger $\alpha \ge 0,\beta>0,\gamma>0$ values are easier to solve, and the corresponding convergence rate becomes faster; 
the rates for specific cases are summarized in Table~\ref{table:rates}.

For unweighted $k$-NN, 
\citet{chaudhuri2014rates} proves the rate $O(n^{-(1+\alpha)\gamma/(2\gamma+d)})$ by imposing $\alpha$-margin condition and $\gamma$-average neighbour smoothness. 
Whereas the rate seems favorable, 
$\gamma$ is upper-bounded by $2$ due to the asymptotic bias, even if highly-smooth function is considered~($\beta \ge 2$; our Theorem~\ref{theo:local_homogeneity_2}). 
Thus the rate for unweighted $k$-NN is $O(n^{-2(1+\alpha)/(4+d)})$ at best.

$k$-NN estimator has much in common with Nadaraya-Watson~(NW) estimator~\citep{tsybakov2009introduction}, and its classifier is proved to attain the same rate $O(n^{-4/(4+d)})$ as unweighted $k$-NN, for $\alpha=1$ and twice-differentiable $\eta$~\citep{hall2005bandwidth}. 
It is also widely known that the convergence rate of local polynomial~(LP)-estimator~\citep{tsybakov2009introduction} is drastically improved from that of the NW-estimator, when approximating highly smooth functions; 
\citet{audibert2007learning} considers a classifier based on LP-estimator and an \emph{uniform bound} of the excess risk over all the possible $\eta$ and the distribution of $X$. 
The rate for LP classifier is $O(n^{-(1+\alpha)\beta/(2\beta+d)})$, 
which is also proved to be optimal among all the classifiers. 
However, the LP estimator employs polynomials of degree $\lfloor \beta \rfloor:=\max\{\beta' \in \mathbb{N}_0 \mid \beta' < \beta\}$; it estimates coefficients of $1+d+d^2+\cdots+d^{\lfloor \beta \rfloor}$ terms, resulting in high computational cost and difficulty in implementation.

Returning back to $k$-NN classifiers, which do not require such a large number of coefficients therein, \citet{samworth2012optimal} finds optimal weights for weighted $k$-NN by minimizing the exact asymptotic expansion of the non-uniform bound of the excess risk. 
When considering only non-negative weights, optimal convergence rate is $O(n^{-4/(4+d)})$, where the rate is still same as the case $\alpha=1$ of the unweighted $k$-NN. 
However, interestingly, \citet{samworth2012optimal} also proves that \textbf{real-valued weights} including negative weights are essential for eradicating the bias and attaining the exact optimal rate $O(n^{-2\beta/(2\beta+d)})$ for 
$\eta\in C^\beta$ with $\alpha=1,\beta=2u \: (u \in \mathbb{N})$.

\textbf{Current issue:} 
In practice, determining the weights explicitly in the way of \citet{samworth2012optimal} is rather burdensome, where explicit weights are shown for limited cases~($\beta=2,4$). 
Other simpler approaches to determine optimal weights are appreciated.

\textbf{Contribution of this paper:} 
We propose \emph{multiscale $k$-NN~(MS-$k$-NN)}, consisting of two simple steps: (1) unweighted $k$-NN estimators are computed for several $k\ge1$ values, and (2) extrapolating them to $k=0$ via regression,
as explained in Figure~\ref{fig:illustration_msknn}.
This algorithm eradicates the asymptotic bias, as it computes an imaginary 0-NN estimator. 
Whereas the MS-$k$-NN is computed quite simply, it corresponds to the weighted $k$-NN equipped with favorable real-valued weights, which are automatically specified via regression.
Our algorithm implicitly computes the optimal weights adaptive to the query and its neighbour points.
We prove that the MS-$k$-NN attains the improved convergence rate $O(n^{-(1+\alpha)\beta/(2\beta+d)})$, that coincides with the optimal rate obtained in \citet{samworth2012optimal} if $\alpha=1,\beta=2u \:(u \in \mathbb{N})$. 
Numerical experiments are conducted for performing MS-$k$-NN.

We last note that, the weights implicitly obtained in MS-$k$-NN are different from those of \citet{samworth2012optimal}, though both of these weights attain the same optimal convergence rate. 
See Figure~\ref{fig:optimal_weights} in Section~\ref{subsec:optimal_real-valued_weights} for the obtained weights. 
Also see Supplement~\ref{supp:related_works} for remaining related works.

\section{Preliminaries}
\label{sec:preliminaries}

We describe the problem setting, notation, and the conditions in Section~\ref{subsec:problem_setting}, \ref{subsec:notation}, \ref{subsec:conditions_assumptions}, respectively.

\subsection{Problem setting}
\label{subsec:problem_setting}
For a non-empty compact set $\mathcal{X} \subset \mathbb{R}^d$,
$d \in \mathbb{N}$, a pair of random variables $(X,Y)$ takes values in $\mathcal{X} \times \{0,1\}$ with joint distribution $\mathbb{Q}$, where 
$X$ represents a feature vector of an object, and 
$Y$ represents its binary class label to which the object belongs. 
$\mu$ represents the probability density function of $X$ and $\eta$ is the conditional expectation
\begin{align}
    \eta(x)=\mathbb{E}(Y \mid X=x).
    \label{eq:conditional_expectation}
\end{align}
Conditions on $\eta$ are listed in Definition~\ref{def:alpha_margin}--\ref{def:gamma_local_homogeneity} later in Section~\ref{subsec:conditions_assumptions}.

$\mathcal{D}_n:=\{(X_i,Y_i)\}_{i=1}^{n}$, $n\in\mathbb{N}$, and $(X_*,Y_*)$ are considered throughout this paper, where they are independent copies of $(X,Y)$; 
$\mathcal{D}_n$ is called a \emph{sample}, and $X_*$ is called a \emph{query}. 
Given a query $X_* \in \mathcal{X}$, we consider predicting the corresponding label $Y_*$ by a \emph{classifier} $\hat{g}_n:\mathcal{X} \to \{0,1\}$ using the sample $\mathcal{D}_n$. 
The performance of a classifier $g$ is evaluated by the misclassification error rate 
$L(g):=\mathbb{P}_{X_*,Y_*}(g(X_*) \neq Y_*)$. 
Under some mild assumptions, \emph{excess risk}
\begin{align}
    \mathcal{E}(\hat{g}_n)
    :=
    \mathbb{E}_{\mathcal{D}_n}(L(\hat{g}_n))-\inf_{g:\mathcal{X} \to \{0,1\}}L(g)
    \label{eq:def_excess_risk}
\end{align}
for various classifiers is proved to approach $0$ as $n \to \infty$. 
Note that the classifier $g_*(X):=\mathbbm{1}(\eta(X) \ge 1/2)$ satisfies $L(g_*)=\inf_{g:\mathcal{X} \to \{0,1\}}L(g)$, and it is said to be Bayes-optimal~\citep[see, e.g.,][Section 2.2]{devroye1996probabilistic}. 
Then, the asymptotic order of the excess risk $\mathcal{E}(\hat{g}_n)$ with respect to the sample size $n$ is called \emph{convergence rate}; 
the goal of this study is to propose a classifier that 
(i) is practically easy to implement, and 
(ii) attains the optimal convergence rate.

\subsection{Notation}
\label{subsec:notation}
For any given query $X_* \in \mathcal{X}(\subset \mathbb{R}^d)$ and a sample $\mathcal{D}_n$, 
the index $1,2,\ldots,n$ is re-arranged to be $(1),(2),\ldots,(n)$ s.t. 
\scalebox{1}{$\displaystyle 
    \|X_*-X_{(1)}\|_{2}
    \le 
    \|X_*-X_{(2)}\|_{2}
    \le 
    \cdots 
    \le 
    \|X_*-X_{(n)}\|_{2}
    $}
where Euclidean norm $\|\bs x\|_{2}=(\sum_{i=1}^{d}x_i^2)^{1/2}$ is employed throughout this paper. 
Note that the re-arranged index $(1),(2),\ldots,(n)$ depends on the query $X_*$; we may also denote the index by $(1;X_*),(2;X_*),\ldots,(n;X_*)$. 
$B(X;r)
:=
\{X' \in \mathcal{X} \mid \|X-X'\|_{2} \le r\}
\subset \mathcal{X}$ 
represents the $d$-dimensional closed ball centered at $\bs x \in \mathcal{X}$ whose radius is $r>0$. 

$f(n) \asymp g(n)$ indicates that the asymptotic order of $f,g$ are the same, 
$\text{tr}\bs A=\sum_{i=1}^{d}a_{ii}$ represents the trace of the matrix $\bs A=(a_{ij}) \in \mathbb{R}^{d \times d}$, $\bs 1=(1,1,\ldots,1)^\top$ is a vector and $\mathbbm{1}(\cdot)$ represents an indicator function. 
$\lfloor \beta \rfloor := \max\{\beta' \in \mathbb{N} \mid \beta' < \beta\}$ for $\beta>0$, $[n]:=\{1,2,\ldots,n\}$ for any $n \in \mathbb{N}$, and 
$\|\bs x\|_{\infty}:=\max_{i \in [d]}|x_i|$ for $\bs x=(x_1,x_2,\ldots,x_d)$. 
Let $\mathbb{N}_0 = \mathbb{N}\cup\{0\}$.
For $q\in\mathbb{N}_0$, $C^{q}=C^{q}(\mathcal{X})$ represents a set of $q$-times continuously differentiable functions $f:\mathcal{X} \to \mathbb{R}$.

\subsection{Conditions}
\label{subsec:conditions_assumptions}

We first list three different types of conditions on the conditional expectation (\ref{eq:conditional_expectation}), 
in Definition~\ref{def:alpha_margin}, \ref{def:beta_holder} and \ref{def:gamma_local_homogeneity} below; they are considered in a variety of existing studies~\citep{gyorfi1981rate,devroye1996probabilistic,audibert2007learning,tsybakov2009introduction,chaudhuri2014rates}.

\begin{defi}[$\alpha$-margin condition]
\label{def:alpha_margin}
If there exist constants $L_{\alpha} \ge 0,\tilde{t}>0$ and $\alpha \ge 0$ such that 
\[
    \mathbb{P}(|\eta(X)-1/2| \le t) \le L_{\alpha} t^{\alpha}, 
\]
for all $t \in (0,\tilde{t}]$ and $X \in \mathcal{X}$, 
$\eta$ is said to be satisfying \emph{$\alpha$-margin condition}, with \emph{margin exponent} $\alpha$. 
\end{defi}

\begin{defi}[$\beta$-H{\"o}lder condition]
\label{def:beta_holder}
Let $\mathcal{T}_{q,X_*}[\eta]$ be the Taylor expansion of a function $\eta$ of degree $q\in\mathbb{N}_0$ at $X_* \in \mathcal{X}$~(See, Definition~\ref{def:taylor} in Supplement for details). 
A function $\eta \in C^{\lfloor\beta\rfloor}(\mathcal{X})$ is said to be \emph{$\beta$-H{\"o}lder}, where $\beta>0$ is called \emph{H{\"o}lder exponent}, if there exists $L_{\beta}>0$ such that 
\begin{align}
|\eta(X)-\mathcal{T}_{\lfloor \beta \rfloor,X_*}[\eta](X)| \le L_{\beta}\|X-X_*\|^{\beta}
\label{eq:beta_holder_audibert}
\end{align}
for any $X,X_* \in \mathcal{X}$. 
Note that a function $\eta\in C(\mathcal{X})^\beta$ for $\beta\in\mathbb{N}$ and compact $\mathcal{X}$ is also $\beta$-H{\"o}lder.
\end{defi}

The above H{\"o}lder condition specifies the smoothness of $\eta$ by a user-specified parameter $\beta>0$, and the above (\ref{eq:beta_holder_audibert}) is employed in many studies, e.g., \citet{audibert2007learning}. It reduces to 
\begin{align}
 |\eta(X)-\eta(X_*)| \le L_\beta\|X-X_*\|^{\beta}
\label{eq:beta-holder-classical}
\end{align}
for $0 < \beta \le 1$. 
However, (\ref{eq:beta_holder_audibert}) and (\ref{eq:beta-holder-classical}) are different for $\beta>1$, where the latter is considered in \citet{chaudhuri2014rates}.

For describing the next condition, we consider $\eta^{(\infty)}(B) :=
\mathbb{E}(Y | X \in B)$,
that is the conditional expectation of $Y$ given $X \in B$ for the set $B \subset \mathbb{R}^d$. $\eta^{(\infty)}$ and the support of $\mu$ are expressed as
\begin{align}
    \eta^{(\infty)}(B)
    &=
    \frac{
        \int_{B \cap \mathcal{X}} \eta(X) \mu(X) \diff X
    }{
        \int_{B \cap \mathcal{X}} \mu(X) \diff X
    },
    \quad 
    \mathcal{S}(\mu)
    :=
    \scalebox{0.95}{$\displaystyle
    \bigg\{X \in \mathcal{X} \mid \int_{B(X;r)}\mu(X) \diff X>0,\forall r>0\bigg\}$}
    \label{eq:smu}
\end{align}
where \citet{chaudhuri2014rates} Lemma~9 proves that $\eta^{(\infty)}(B(X_*;r))$ for $X_* \in S(\mu)$ asymptotically approximates the unweighted $k$-NN estimator (with roughly $r=\|X_{(k)}-X_*\|_2$), that will be formally defined in Definition~\ref{def:weighted_knn}. 
$\mathcal{S}(\mu)$ is assumed to be compact throughout this paper.

\begin{defi}[$\gamma$-neighbour average smoothness]
\label{def:gamma_local_homogeneity}
If there exists $L_{\gamma},\gamma>0$ such that 
\begin{align*}
    |\eta^{(\infty)}(B(X;r))-\eta(X)|
    \le 
    L_{\gamma} r^{\gamma}
\end{align*}
for all $r>0$ and $X \in \mathcal{S}(\mu)$, then the function $\eta$ is said to be \emph{$\gamma$-neighbour average smooth} with respect to $\mu$, where $\gamma$ is called \emph{neighbour average exponent}. 
A weaker version of this condition is used in \citet{gyorfi1981rate}, where
the constant $L_\gamma$ is replaced by a function $L_\gamma(X)$.
A related but different condition called ``$(\alpha, L)$-smooth'' is used in \citet{chaudhuri2014rates}; see Supplement~\ref{supp:note_on_chaudhuri}.
\end{defi}

We last define an assumption on the density of $X$, that is employed in \citet{audibert2007learning}. 

\begin{defi}[Strong density assumption]
If there exist $\mu_{\min},\mu_{\max} \in (0,\infty)$ such that $\mu_{\min} \le \mu(X) \le \mu_{\max}$ for all $X \in \mathcal{X}$, $\mu$ is said to be satisfying \emph{strong density assumption}. 
\end{defi}

\section{$k$-NN classifier and convergence rates}
\label{sec:existing_classifiers_and_convergence_rates}

In Section~\ref{subsec:plug-in_classifiers_list}, we first define $k$-NN classifier. 
Subsequently, we review existing studies on convergence rates for unweighted $k$-NN and weighted $k$-NN classifiers in Section~\ref{subsec:unweighted_knn} and \ref{subsec:weighted_knn}, respectively. 
Other classifiers and their convergence rates are also presented in Supplement~\ref{supp:LP}.

\subsection{$k$-NN classifier}
\label{subsec:plug-in_classifiers_list}
In this paper, we consider only a plug-in classifier~\citep{audibert2007learning}
\begin{align}
    g^{\text{(plug-in)}}(X;\hat{\eta}_n):=\mathbbm{1}(\hat{\eta}_n(X) \ge 1/2),
    \label{eq:def_plug_in_classifier}
\end{align}
where $\hat{\eta}_n(X)$ is an estimator of $\eta(X)$, that leverages the sample $\mathcal{D}_n$. 
Given a query $X_* \in \mathcal{X}$, 
an archetypal example of the function value $\hat{\eta}_n(X_*)$ is in the following.

\begin{defi}
\label{def:weighted_knn}
Weighted $k$-NN estimator is defined as 
\begin{align}
    \hat{\eta}_{n,k,\bs w}^{(k\text{NN})}(X_*)
    :=
    \sum_{i=1}^{k} w_i Y_{(i;X_*)},
    \label{eq:weighted_knn}
\end{align}
where $(1;X_*),(2;X_*),\ldots,(n;X_*)$ is the re-arranged index defined in Section~\ref{subsec:notation} and $k \in \mathbb{N}$ is a user-specified parameter. It is especially called unweighted $k$-NN if $w_1=w_2=\cdots=w_k=1/k$, and is denoted by $\hat{\eta}_{n,k}^{(k\text{NN})}$. 
The (weighted) $k$-NN classifier is $\hat{g}_{n,k,\bs w}^{(k\text{NN})}(X):=g^{(\text{plug-in})}(X;\hat{\eta}_{n,k,\bs w}^{(k\text{NN})})$. 
\end{defi}

\subsection{Convergence rate for unweighted $k$-NN classifier}
\label{subsec:unweighted_knn}

Here, we consider the unweighted $k$-NN; the following Proposition~\ref{prop:chaudhuri_theorem4} shows the convergence rate.

\begin{prop}[A slight modification of \citet{chaudhuri2014rates} Th.~4]
\label{prop:chaudhuri_theorem4}
Let $\mathcal{X}$ be a compact set, and assuming that 
(i) $\eta$ satisfies $\alpha$-margin condition and is $\gamma$-neighbour average smooth, and (ii) $\mu$ satisfies strong density assumption. 
Then, the convergence rate of the unweighted $k$-NN classifier with $k_* = k_n \asymp n^{2\gamma/(2\gamma+1)}$ is 
\[
    \mathcal{E}(\hat{g}_{n,k_*}^{(k\text{NN})})
    =
    O(n^{-(1+\alpha)\gamma/(2\gamma+d)}).
\]
\end{prop}
\begin{proof} \citet{chaudhuri2014rates} Theorem~4(b) shows the convergence rate; see Supplement~\ref{supp:note_on_chaudhuri} for the correspondence of the assumption and symbols. 
\end{proof}

Our current concern is whether the convergence rate $O(n^{-(1+\alpha)\gamma/(2\gamma+d)})$ of the unweighted $k$-NN classifier can be associated to the rate $O(n^{-(1+\alpha)\beta/(2\beta+d)})$ of the LP classifier, 
whose optimality is proved by \citet{audibert2007learning} and is formally described in Proposition~\ref{prop:audibert2007_lp} in Supplement~\ref{supp:LP}. 
\citet{chaudhuri2014rates} asserts that these rates are the same, i.e., $\gamma=\beta$, if there exists $L_\beta>0$ such that (\ref{eq:beta-holder-classical}) holds for any $X, X_* \in \mathcal{X}$.
However, only constant functions can satisfy the condition~(\ref{eq:beta-holder-classical}) for $\beta>1$~\citep[][Lemma~2.3]{mittmann2003existence};  only an extremely restricted function class is considered in \citet{chaudhuri2014rates}.

We here return back to the $\beta$-H{\"o}lder condition (\ref{eq:beta_holder_audibert}) considered in this paper and \citet{audibert2007learning}, 
that is compatible with the condition~(\ref{eq:beta-holder-classical}) for $\beta \le 1$ but is different for $\beta>1$. 
Whereas a variety of functions besides constant functions satisfy the $\beta$-H{\"o}lder condition (\ref{eq:beta_holder_audibert}), 
our following Theorem~\ref{theo:local_homogeneity_2} shows that $\gamma = 2$ even if $\eta$ is highly smooth~($\beta \gg 2$).  
Especially for $\alpha=1$, the rate of unweighted $k$-NN coincides with the rate $O(n^{-4/(4+d)})$ of NW-classifier~\citep{hall2005bandwidth}.

\begin{theo}
\label{theo:local_homogeneity_2}
Let $\mathcal{X}$ be a compact set, and let $\beta>0$. 
Assuming that 
(i) $\mu$ and $\eta \mu$ are $\beta$-H{\"o}lder, and 
(ii) $\mu$ satisfies the strong density assumption, 
there exist constants $L_{\beta}^*>0,\tilde{r}>0$ and continuous functions $b_1^*,b_2^*,\ldots,b_{\lfloor \beta/2 \rfloor}^*,\delta_{\beta, r}:\mathcal{X} \to \mathbb{R}$ such that
\begin{align*}
    \eta^{(\infty)}(B(X_*;r)) - \eta(X_*) &= 
    \scalebox{0.9}{$\displaystyle\sum_{c=1}^{\lfloor \beta/2 \rfloor}b_c^*(X_*) r^{2c}
    +
    \delta_{\beta,r}(X_*)$}, \quad 
    |\delta_{\beta,r}(X_*)| \le L_{\beta}^* r^{\beta} 
\end{align*}
for all $r \in (0,\tilde{r}],X_* \in \mathcal{S}(\mu)$ defined in (\ref{eq:smu}). 
For $\beta>2$, 
$b_1^*(X_*) =
    \frac{1}{2d+4}
    \frac{1}{\mu(X_*)}
    \{\Delta [\eta(X_*) \mu(X_*)]-\eta(X_*)\Delta \mu(X_*)\}$ with $\Delta:=\frac{\partial^2}{\partial x_1^2}+\cdots+\frac{\partial^2}{\partial x_d^2}$; 
if $\eta$ is $\beta(\ge 2)$-H{\"o}lder, $\eta$ is $(\gamma=)2$-neighbour average smooth and 
\begin{align*}
\mathcal{E}(\hat{g}_{n,k_*}^{(k\text{NN})})=O(n^{-2(1+\alpha)/(4+d)}).
\end{align*}
\end{theo}
\begin{proof}
The numerator and denominator of $\eta^{(\infty)}(B(X_*;r))$ are obtained via integrating Taylor expansions of $\eta \mu$ and $\mu$, respectively; division proves the assertion. 
See, Supplement~\ref{supp:proof_prop:local_homogeneity_2} for details.
\end{proof}

\subsection{Convergence rate for weighted $k$-NN classifier}
\label{subsec:weighted_knn}

Here, we consider the weighted $k$-NN. 
\citet{samworth2012optimal} first derives non-negative optimal weights
\begin{align}
\scalebox{0.9}{$\displaystyle
    w_i^*
    :=
    \frac{1}{k_*}\left\{ 1+\frac{d}{2}-\frac{d}{2k_*^{2/d}}(i^{1+2/d}-(i-1)^{1+2/d}) \right\} 
    $}
    \label{eq:samworth_optimal_nonneg}
\end{align}
for $i \in [k_*]$ and $0$ otherwise, where $k_* \asymp n^{4/(d+4)}$, 
through the asymptotic expansion of the excess risk. 
However, the obtained rate is still $O(n^{-4/(4+d)})$~(\citet{samworth2012optimal} Theorem~2), that is the same as the case $\alpha=1$ of unweighted $k$-NN~(Theorem~\ref{theo:local_homogeneity_2}); convergence evaluation of the $k$-NN still remains slow, even if arbitrary weights can be specified.

For further improving the convergence rate, \citet{samworth2012optimal} also considers \textbf{real-valued weights} allowing negative values. The improved convergence rate is given in the following Proposition~\ref{prop:samworth_real_valued_weights}. 
Formal descriptions, i.e., definition of the weight set $\mathcal{W}_{n,s}$ and conditions for their rigorous proof, are described in Supplement~\ref{supplement:samworth} due to the space limitation. 

\begin{prop}[\citet{samworth2012optimal} Th.~6]
\label{prop:samworth_real_valued_weights}
    Let $\mathcal{W}_{n,s}$ be a set of real-valued weights defined in Supplement~\ref{supplement:samworth}, where
    we assume the conditions (i)--(iv) therein.
    Note that the condition~(ii) implies
    $\eta \in C^{\beta}$, $\beta=2u$ for $u \in \mathbb{N}$.
    If $(\alpha=)1$-margin condition is assumed, 
    then
    \begin{align}
    \mathcal{E}(\hat{g}^{(k\text{NN})}_{n,k,\bs w})
    \asymp
        \left\{
        B_1 \sum_{i=1}^{n} w_i^2 + B_2\left(
            \sum_{i=1}^{n}\frac{\delta_i^{(u)} w_i}{n^{2r/d}}
        \right)^2
        \right\} (1+o(1))
    \label{eq:excess_risk_weighted_knn_samworth}
    \end{align}
    holds for $\bs w \in \mathcal{W}_{n,s}$, where $B_1,B_2$ are some constants and $\delta_i^{(\ell)}:=i^{1+2\ell/d}-(i-1)^{1+2\ell/d} \: (\forall \ell \in [u])$. 
\end{prop}

Following Proposition~\ref{prop:samworth_real_valued_weights}, \citet{samworth2012optimal} shows that the asymptotic minimizer of the excess risk (\ref{eq:excess_risk_weighted_knn_samworth}) with the weight constraint $\sum_{i=1}^{n}w_i=1,\sum_{i=1}^{n}\delta_i^{(\ell)}w_i=0$ for all $\ell \in [u-1]$, and $w_i=0$ for $i=k^*+1,\ldots,n$ with $k^* \asymp n^{2\beta/(2\beta+d)} $ is in the form of 
\begin{align}
    w_{i}^*
    :=
    (a_0 + a_1 \delta_i^{(1)}+\cdots+a_u \delta_i^{(u)})/k_*
    \label{eq:samworth_optimal_weights}
\end{align}
for $i=1,2,\ldots,k_*$, where
$\bs a=(a_0,a_1,\ldots,a_u) \in \mathbb{R}^{u+1}$ are unknowns. 
\citet{samworth2012optimal} proposes to find $\bs a$ so that (\ref{eq:samworth_optimal_weights}) satisfies the weight constraint. 
Then, the following optimal rate is obtained.

\begin{coro}
\label{coro:samworth_real_valued_weights}
Symbols and assumptions are the same as those of Proposition~\ref{prop:samworth_real_valued_weights}. 
Then, the optimal $\bs w_*$ and $k_* \asymp n^{2\beta/(2\beta+d)}$ lead to
    \begin{align*}
        \mathcal{E}(\hat{g}^{(k\text{NN})}_{n,k_*,\bs w_*})
        \asymp 
        n^{-2\beta/(2\beta+d)}. 
    \end{align*}
\end{coro}
Although only the case $\alpha=1$ is considered in \citet{samworth2012optimal}, the convergence rate in Corollary~\ref{coro:samworth_real_valued_weights} coincides with the rate for LP-classifier, given in Proposition~\ref{prop:audibert2007_lp}.

Whereas theories can be constructed without solving the equations, 
solving the equations to determine the optimal real-valued weights explicitly is rather burdensome, where the explicit solution is shown only for $u=1,2$ (namely, $\beta=2,4$) in \citet{samworth2012optimal}; the solution for $u=2$ is
$a_1
    :=
    \frac{1}{(k_*)^{2/d}}\left\{
        \frac{(d+4)^2}{4}
        -
        \frac{2(d+4)}{d+2}a_0
    \right\},
    a_2
    =
    \frac{1-a_0-(k_*)^{2/d}a_1}{(k_{*})^{4/d}}
    $ 
(see, Supp.~\ref{supplement:samworth} for more details, and also see Figure~\ref{fig:samworth} in Supp.~\ref{supp:optimal_weights_msknn} for the optimal weights computed in an experiment~($u=2$)).

We also note that, conducting cross-validation to choose $(w_1,w_2,\ldots,w_k)$ directly from $\mathcal{W}^k$ (for some sets $\mathcal{W} \subset \mathbb{R}$) is impractical, as it requires large computational cost $O(|\mathcal{W}|^k)$. 
Therefore, other simpler approaches to determine optimal real-valued weights are appreciated in practice.

\section{Proposed multiscale $k$-NN}
\label{sec:knn}

In this section, we propose multiscale $k$-NN~(MS-$k$-NN), that implicitly finds favorable real-valued weights for weighted $k$-NN. 
Note that the obtained weights are different from \citet{samworth2012optimal}, as illustrated in Figure~\ref{fig:optimal_weights} in Supplement~\ref{supp:optimal_weights_msknn}. 
In what follows, we first formally define MS-$k$-NN in Section~\ref{subsec:msknn}. 
Subsequently, the weights obtained via MS-$k$-NN are shown in Section~\ref{subsec:optimal_real-valued_weights}, 
the convergence rate is discussed in Section~\ref{subsec:convergence_rate_msknn}.

\subsection{Multiscale $k$-NN}
\label{subsec:msknn}

\paragraph{Underlying idea:} 
Since $\eta^{(\infty)}(B(X_*;r))$ asymptotically approximates the $k$-NN estimator $\hat{\eta}^{(k\text{NN})}_{n,k}(X_*)$ for roughly $r=r(k):=\|X_{(k)}-X_*\|$~(see, e.g., \citet{chaudhuri2014rates} Lemma~9), 
asymptotic expansion in Theorem~\ref{theo:local_homogeneity_2} indicates that
$\hat{\eta}^{(k\text{NN})}_{n,k}(X_*)
\approx 
\eta(X_*)+\sum_{c=1}^{\lfloor \beta/2 \rfloor} b_c^* r^{2c}$, 
for some $\{b_c^*\} \subset \mathbb{R}$. 
Estimating a function $f_{X_*}(r):=b_0 + \sum_{c=1}^{\lfloor \beta/2 \rfloor}b_c r^{2c}$ to predict the $k$-NN estimators for $k_1,k_2,\ldots,k_V$ and extrapolating to $k=0$ via $r=r(k)$ with $r(0):=0$ yields $\hat{f}_{X_*}(0)=\hat{b}_0 \approx \eta(X_*)$; the asymptotic bias $\sum b_c^* r^{2c}$ is then eradicated.

\paragraph{Definition of MS-$k$-NN:} 
Let $V,C\in\mathbb{N}$ and fix any query $X_* \in \mathcal{X}$. 
We first compute unweighted $k$-NN estimators for $1 \le k_1 < k_2 < \cdots < k_{V} \le n$, i.e., 
$\hat{\eta}^{(k\text{NN})}_{n,k_1}(X_*),
    \hat{\eta}^{(k\text{NN})}_{n,k_2}(X_*), 
    \ldots,
    \hat{\eta}^{(k\text{NN})}_{n,k_V}(X_*)$. Then, we compute $r_v:=\|X_{(k_v)}-X_*\|_2$, 
and consider a simple regression such that
\begin{align}
    \hat{\eta}^{(k\text{NN})}_{n,k_v}(X_*)
    \approx 
    b_0
    +
    b_1 r_v^2
    +
    b_2 r_v^4
    +
    \cdots 
    +
    b_{C} r_v^{2C}
    \label{eq:msknn_regression_function}
\end{align} 
for all $v \in [V]$, where $\bs b=(b_0,b_1,\ldots,b_C)$ is a regression coefficient vector to be estimated. 
Note that the regression function is a polynomial of $r_v^2$ which contains only terms of even degrees $r_v^{2c}$, since all the bias terms are of even degrees as shown in Theorem~\ref{theo:local_homogeneity_2}. 
However, it is certainly possible that we employ a polynomial with terms of odd degrees in practical cases.

More formally, we consider a minimization problem
\begin{align}
    \scalebox{0.95}{$\displaystyle
    \hat{\bs b}
    :=
    \argmin_{\bs b \in \mathbb{R}^{C+1}}
    \sum_{v=1}^{V}
    \left(
        \hat{\eta}^{(k\text{NN})}_{n,k_v}(X_*)
        -
        b_0
        -
        \sum_{c=1}^{C}
        b_c r_v^{2c}
    \right)^2$}. 
    \label{eq:msknn_estimator}
\end{align}
Then, we propose a \emph{multiscale $k$-NN~(MS-$k$-NN)} estimator
\begin{align}
    \hat{\eta}^{\text{(MS-$k$NN)}}_{n,\bs k}(X_*)
    :=
    \hat{b}_{0}
    \quad 
    \scalebox{0.9}{$\displaystyle\left(=
    \bs z(X_*)^{\top}\hat{\bs \eta}^{(k\text{NN})}_{n,\bs k}(X_*) \right)$},
    \label{eq:msknn_definition}
\end{align}
where $\bs k=(k_1,k_2,\ldots,k_V) \in \mathbb{N}^V$, 
$\hat{\bs \eta}^{(k\text{NN})}_{n,\bs k}(\bs X_*)
:=
(\hat{\eta}^{(k\text{NN})}_{n,k_1}(X_*),\ldots,\hat{\eta}^{(k\text{NN})}_{n,k_V}(X_*))^\top \in \mathbb{R}^V$ and
 $\bs z(X_*) \in \mathbb{R}^V$ will be defined in (\ref{eq:explicit_z}). 
 Since (\ref{eq:msknn_regression_function}) extrapolates $k$-NN estimators to $r=0$, we also call the situation by ``extrapolating to $k=0$" analogously. 
The corresponding \emph{MS-$k$-NN classifier} is defined as 
\begin{align}
\hat{g}^{\text{(MS-$k$NN)}}_{n,\bs k}(X)
:=
g^{(\text{plug-in})}(X;\hat{\eta}^{\text{(MS-$k$NN)}}_{n,\bs k}).
\label{eq:msknn_classifier}
\end{align}

Note that the number of terms in the regression function (\ref{eq:msknn_regression_function}) is $1+C$, and 
$C$ will be specified as $C=\lfloor \beta/2 \rfloor$ under the $\beta$-H{\"o}lder condition in Theorem~\ref{theo:msknn_cr}. 
Although the parameter $\beta$ cannot be observed in practice, we may employ large $C$ so that $C \ge \lfloor \beta/2 \rfloor$ is expected~(e.g., $C=10$). Even in this case, overall number of terms in the MS-$k$-NN is $1+C$, which is much less than the number of coefficients used in the LP classifier $(=1+d+d^2+\cdots+d^{C})$.

\subsection{Corresponding real-valued weights}
\label{subsec:optimal_real-valued_weights}

In this section, real-valued weights implicitly obtained via MS-$k$-NN are considered. 
The vector $\bs z(X_*)=(z_1(X_*),z_2(X_*),\ldots,z_V(X_*))^\top \in \mathbb{R}^{V}$ in the definition of MS-$k$-NN~(\ref{eq:msknn_definition}) is obtained by simply solving the minimization problem~(\ref{eq:msknn_estimator}), as
\begin{align}
\bs z(X_*)
:=
\frac{(\bs I-\mathcal{P}_{\bs R(X_*)})\bs 1}{V-\bs 1^{\top}\mathcal{P}_{\bs R(X_*)}\bs 1},
    \label{eq:explicit_z}
\end{align}
where $\bs 1=(1,1,\ldots,1)^\top \in \mathbb{R}^V,\mathcal{P}_{\bs R}=\bs R(\bs R^{\top}\bs R)^{-1}\bs R$ and $(i,j)$-th entry of $\bs R=\bs R(X_*)$ is $r_i^{2j}$ for $(i,j) \in [V] \times [C]$; note that the radius $r_i$ depends on the query $X_*$. 
Therefore, the corresponding optimal real-valued weight $\bs w_*(X_*)=(w_1^*(X_*),w_2^*(X_*),\ldots,w_{k_V}^*(X_*))$ is obtained as
\begin{align}
    w_{i}^*(X_*)
    :=
    \sum_{v:i \le k_v}\frac{z_v(X_*)}{k_v} \in \mathbb{R},
    \quad (\forall i \in [k_V]), 
    \label{eq:optimal_weights_msknn}
\end{align}
then $\hat{\eta}^{(k\text{NN})}_{n,k_V,\bs w_*}(X_*)=\hat{\eta}^{(\text{MS-$k$NN})}_{n,\bs k}(X_*)$. 
Here, we note that the weight (\ref{eq:optimal_weights_msknn}) is adaptive to the query $X_*$, as each entry of the matrix $\bs R$ used in the definition of $\bs z$~(\ref{eq:explicit_z}) depends on both sample $\mathcal{D}_n$ and query $X_*$.
See Supplement~\ref{supp:optimal_weights_msknn} for the skipped derivation of the above (\ref{eq:explicit_z}) and (\ref{eq:optimal_weights_msknn}). 
Total sum of the weights~(\ref{eq:optimal_weights_msknn}) is then easily proved as
$\sum_{i=1}^{k_V}w_i^*(X_*)
=
\sum_{v=1}^{V}z_v(X_*)
=
\bs 1^{\top}\bs z(X_*)
\overset{\text{(\ref{eq:explicit_z})}}{=}
1$.

To give an  example, we plotted the weights in the following Figure~\ref{fig:optimal_weights}. 
The real-valued weights implicitly computed in MS-$k$-NN are plotted to compare with the optimal real-valued weights proposed by \citet{samworth2012optimal}.

\begin{figure}[htbp]
\centering
\subfigure[\citet{samworth2012optimal}]{
\includegraphics[scale=0.4]{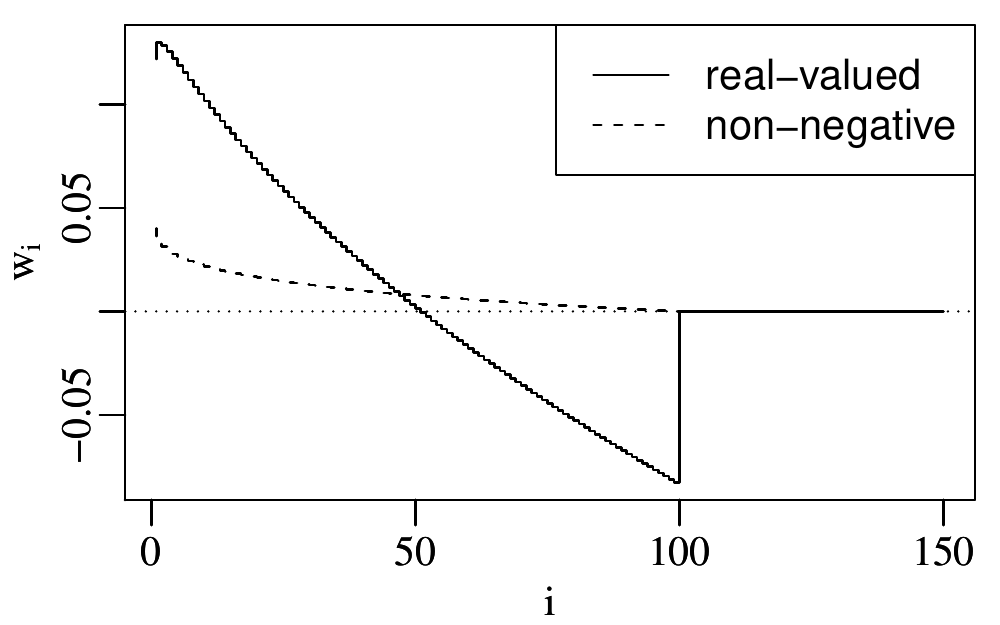}
    \label{fig:samworth}
}
\subfigure[Proposal~($V=5$)]{
\includegraphics[scale=0.4]{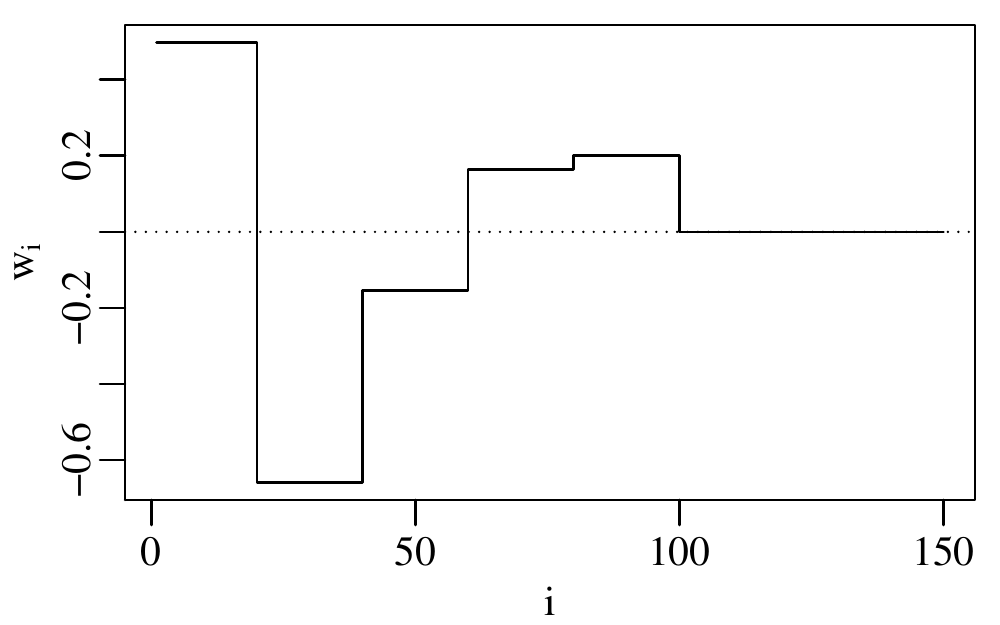}
	\label{fig:proposed_v5}
}
\subfigure[Proposal~($V=10$)]{
\includegraphics[scale=0.4]{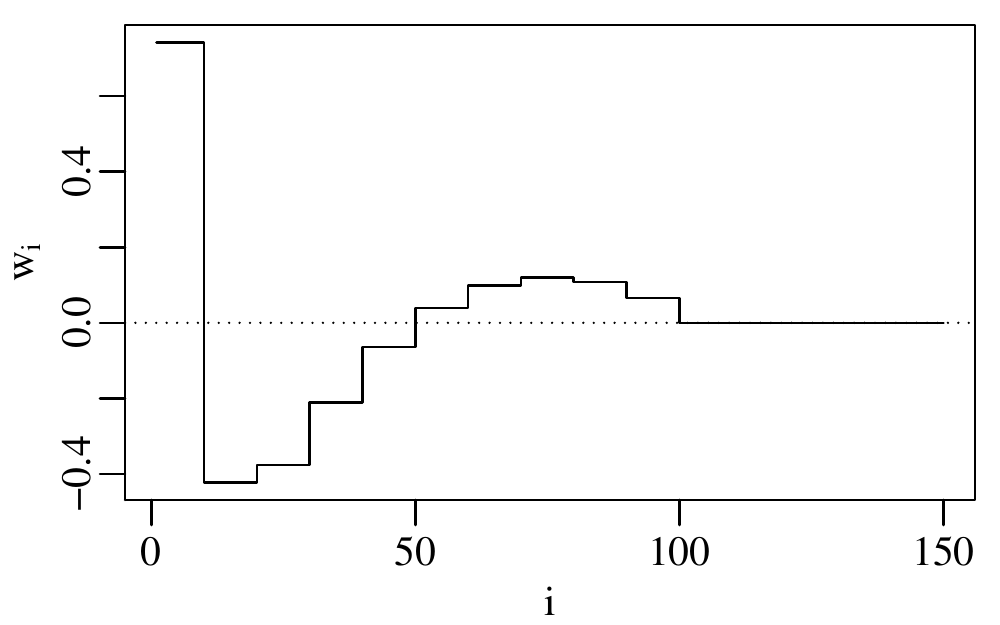}
	\label{fig:proposed_v10}
}
\vspace{-0.5em}
\caption{
Amongst all the experiments, $n=1000,d=10,u=C=2,k_*=100$. 
In \subref{fig:samworth}, optimal non-negative~(\ref{eq:samworth_optimal_nonneg}) and real-valued~(\ref{eq:samworth_optimal_weights}) weights for weighted $k$-NN~(\ref{eq:weighted_knn}) in \citet{samworth2012optimal} are plotted. 
In \subref{fig:proposed_v5} and \subref{fig:proposed_v10}, real-valued weights~(\ref{eq:optimal_weights_msknn}) implicitly computed in the \textbf{proposed MS-$k$-NN}, are plotted for $k_v=k_*v/V,r_v:=(k_v/n)^{1/d} \: (v \in [V])$. 
$V$ is the number of $k$ used for regression. 
}
\label{fig:optimal_weights}
\end{figure}

Figure~\ref{fig:proposed_v5} and \ref{fig:proposed_v10} illustrate the optimal weights (\ref{eq:optimal_weights_msknn}) for $V=5,10$. 
The weights are not monotonically decreasing for $i \le k_V(=100)$, and the weights are smoothly connected to $w^*_{k_V+1}=0$ at $i \approx k_V$, unlike \citet{samworth2012optimal} shown in Figure~\ref{fig:samworth}.

Although the weights~(\ref{eq:optimal_weights_msknn}) can be easily computed, they are not computed in practice. 
Only a procedure needed for MS-$k$-NN is to
conduct the regression~(\ref{eq:msknn_estimator}) and specify $\hat{\eta}^{(\text{MS-$k$NN})}_{n,k}$ by the intercept $\hat{b}_0$ stored in the regression coefficient $\hat{\bs b}$. 
Then, MS-$k$-NN automatically coincides with the weighted $k$-NN using the above optimal weight~(\ref{eq:optimal_weights_msknn}).

\subsection{Convergence rate for MS-$k$-NN classifier}
\label{subsec:convergence_rate_msknn}

Here, we consider the convergence rate for MS-$k$-NN classifier. Firstly, we specify a vector $\bs \ell=(\ell_1,\ell_2,\ldots,\ell_V) \in \mathbb{R}^V$ so that $\ell_1=1<\ell_2<\cdots<\ell_V<\infty$. We assume that 
\begin{itemize}
\item[(C-1)] $k_{1,n} \asymp n^{2\beta/(2\beta+d)}$, 
\item[(C-2)] $k_{v,n}:=\min\{k \in [n] \mid \|X_{(k)}-X_*\|_2 \ge \ell_v r_{1,n}\}$ for $v=2,3,\ldots,V$, where $r_{1,n}:=\|X_{(k_{1,n})}-X_*\|_2$, 
\item[(C-3)] $\exists L_{\bs z}>0$ such that $\|\bs z_{\bs \ell}\|_{\infty}<L_{\bs z}$, where $\bs z_{\bs \ell}=\frac{(\bs I-\mathcal{P}_{\bs R})\bs 1}{\bs 1^{\top}(\bs I-\mathcal{P}_{\bs R})\bs 1}$ and $\bs R=(\ell_{i}^{2j})_{ij} \in \mathbb{R}^{[V] \times [C]}$, 
for all $X_* \in \mathcal{S}(\mu)$. 
\end{itemize}

Intuition behind the above conditions (C-1)--(C-3) is as follows. 
(C-1): local polynomial classifier with bandwidth $h=h_n \asymp n^{-1/(2\beta+d)}$ is known to attain the optimal rate~\citep[][Theorem~3.3]{audibert2007learning}. Therefore, the information in the ball $B(X;h)$ with radius $h>0$ is roughly required for the optimal rate. When assuming that the feature vectors $X_1,X_2,\ldots,X_n$ distribute uniformly, 
$k$ and the bandwidth $h$ have the relation $k \asymp nh^d \asymp n^{2\beta/(2\beta+d)}$ as the volume of $B(X;h)$ is of order $h^d$. 
(C-2): $k_{1,n},k_{2,n},\ldots,k_{v,n}$ are selected so that $r_{2,n},\ldots,r_{v,n}$ are in the same order as $r:=r_{1,n}$ (and $r_{v,n}/r_{1,n} \to \ell_v$).
(C-3): the weights $w_1,w_2,\ldots,w_k$ estimated via regression take finite values asymptotically. 

Then, regarding the MS-$k$-NN estimator~(\ref{eq:msknn_definition}) and its corresponding MS-$k$-NN classifier~(\ref{eq:msknn_classifier}), the following Theorem~\ref{theo:msknn_cr} holds. 

\begin{theo}[Convergence rate for MS-$k$-NN]
\label{theo:msknn_cr}
Assuming that (i) $\mu$ and $\eta \mu$ are $\beta$-H{\"o}lder, (ii) $\mu$ satisfies the strong density asumption, 
(iii) $C:=\lfloor \beta/2 \rfloor \le V-1$, and 
(iv) the conditions (C-1)--(C-3) are satisfied. 
Then,  
\begin{align*}
    \mathcal{E}(\hat{g}^{(\text{MS-$k$NN})}_{n,\bs k})
    =
    O(n^{-(1+\alpha)\beta/(2\beta+d)}).
\end{align*}
\end{theo}

\textit{Sketch of Proof:}
By following the underlying idea explained at first in Section~\ref{subsec:msknn},  the bias $O(r^{\min\{2,\beta\}})$ of the conventional $k$-NN is reduced to $O(r^{\beta})$. 
Therefore, intuitively speaking, replacing the bias term in the proof of \citet{chaudhuri2014rates} Theorem~4(b) leads to the proof of Theorem~\ref{theo:msknn_cr}. 
Although this proof stands on the simple underlying idea, technically speaking, some additional considerations are needed; see Supplement~\ref{supp:proof_msknn_cr} for detailed proof. 
\qed

The rate obtained in Theorem~\ref{theo:msknn_cr} coincides with the optimal rate provided in Corollary~\ref{coro:samworth_real_valued_weights}; MS-$k$-NN is an optimal classifier, at least for the case $\alpha=1,\beta=2u \: (u \in \mathbb{N})$.

\subsection{Using $\log k$ as the predictor}
\label{subsec:msknn_direct}

The standard MS-$k$-NN predicts unweighted $k$-NN estimators through the radius $r=r(k)$, that is computed via sample $\mathcal{D}_n$. 
As an alternative approach, we instead consider predicting the estimators directly from $k$.

For clarifying the relation between the radius $r=r(k)$ and $k$, we here consider the simplest setting that the feature vector $X$ distributes uniformly. 
Then, $r_v:=\|X_{(k_v)}-X_*\|_2$ used in (\ref{eq:msknn_regression_function}) is roughly proportional to $k_v^{1/d}$ since the volume of the ball of radius $r_v$ is proportional to $r_v^d$.

Then, for sufficiently large $d$, 
\begin{align}
    r_v^2
    &\propto
    k_v^{2/d} 
    =
    \exp\left(
        \frac{2}{d} \log k_v
    \right)
    =
    1 + \frac{2}{d}\log k_v + O(d^{-2}). 
    \label{eq:high_dim}
\end{align}
Thus, (\ref{eq:msknn_regression_function}) can be expressed as a polynomial with respect to  $\log k_v$ instead of $r_v^2$.
In numerical experiments, we then extrapolate unweighted $k$-NN to  $k=1$.

\section{Numerical experiments}
\label{sec:experiments}

\textbf{Datasets:} 
We employ 
datasets from UCI Machine Learning Repository~\citep{dua2017uci}. 
Each of datasets consists of $d$-dimensional $n$ feature vectors $X_i \in \mathcal{X}$, and their labels $Y_i \in \{1,2,\ldots,m\}$ representing $1$ of $m$ categories.

\textbf{Preprocessing:} 
Feature vectors are first normalized, and then randomly divided into $70\%$ for prediction ($n_{\text{pred}}=\lfloor 0.7n \rfloor$) and the remaining for test query.

\textbf{Evaluation metric:} 
Category of the query is predicted so that the corresponding estimator attains the maximum value. 
The classification accuracy is evaluated via $10$ times experiments. 
The MS-$k$-NN estimated via radius $r(k)$ and that via $\log k$ described in Section~\ref{subsec:msknn_direct} are performed with $C=1$; they are compared with unweighted $k$-NN and weighted $k$-NN with the optimal non-negative and real-valued weights~\citep{samworth2012optimal}. 
Regression in MS-$k$-NN is ridge regularized with the coefficient $\lambda=10^{-4}$.

\textbf{Parameter tuning:} 
For unweighted and weighted $k$-NN, we first fix $k:=V \cdot \lfloor n_{\text{pred}}^{4/(4+d)}\rfloor \asymp n_{\text{pred}}^{4/(4+d)}$. 
Using the same $k$, we simply choose $k_1:=k/V,k_2=2k/V,\ldots,k_V=k$ with $V=5$ for MS-$k$-NN.

\textbf{Results:} 
Sample mean and the sample standard deviation on $10$ experiments are shown in Table~\ref{tab:experiments_result}. 
Overall, weighted $k$-NN and MS-$k$-NN show better score than unweighted $k$-NN~($w_i=1/k$). 
MS-$k$-NN via radius $r(k)$ shows the best or second best score for (all of) $13$ datasets; this number is maximum among all the methods considered in these experiments. 
As well as the MS-$k$-NN using $r(k)$, MS-$k$-NN using the alternative predictor $\log k$ also shows promising scores for some datasets. 
Regarding larger datasets such as MAGIC and Avila, weighted $k$-NN equipped with real-valued weights, which are computed by both ways of \citet{samworth2012optimal} and MS-$k$-NN, demonstrate slightly better performance than the weighted $k$-NN with non-negative weights; this observation coincides with the theoretical optimality.

\begin{table}[htbp]
\centering
\caption{Each dataset consists of $n$ feature vectors whose dimension is $d$; each object is labeled by $1$ of $m$ categories. 
Sample average and the standard deviation for the prediction accuracy are computed on $10$ times experiments. 
Best scores are \first{bolded}, and second best scores are \second{underlined}.}
\label{tab:experiments_result}
\begin{adjustbox}{scale=0.82}
\begin{tabular}{lcccccccc}
\toprule[0.2ex]
\multirow{2}{*}{Dataset} & \multirow{2}{*}{$n$} & \multirow{2}{*}{$d$} & \multirow{2}{*}{$m$} & \multicolumn{3}{c}{$k$-NN} & \multicolumn{2}{c}{\textbf{MS-$k$-NN}} \\
\cmidrule(lr){5-7}\cmidrule(lr){8-9}
& & & & $w_i=1/k$ & $w_i \ge 0$~(\ref{eq:samworth_optimal_nonneg}) & $w_i \in \mathbb{R}$~(\ref{eq:samworth_optimal_weights}) & via $r(k)$~(\ref{eq:msknn_definition}) & via $\log k$~(\ref{eq:high_dim}) \\
\hline 
Iris & 150 & 4 & 3 & $0.83\pm0.04$ & $0.92\pm0.05$ & $0.92\pm0.04$ & $\second{0.93}\pm0.04$ & $\first{0.96}\pm0.04$ \\ 
Glass identification & 213 & 9 & 6 & $0.58\pm0.06$ & $\second{0.64}\pm0.06$ & $\first{0.67}\pm0.05$ & $\second{0.64}\pm0.05$ & $\second{0.64}\pm0.05$ \\ 
Ecoli & 335 & 7 & 8 & $0.80\pm0.03$ & $\first{0.85}\pm0.03$ & $\second{0.84}\pm0.02$ & $\first{0.85}\pm0.02$ & $\second{0.84}\pm0.02$ \\ 
Diabetes & 768 & 8 & 2 & $\first{0.75}\pm0.03$ & $\second{0.74}\pm0.03$ & $0.70\pm0.04$ & $\first{0.75}\pm0.03$ & $0.71\pm0.03$ \\ 
Biodegradation & 1054 & 41 & 2 & $\second{0.84}\pm0.02$ & $\first{0.86}\pm0.03$ & $0.79\pm0.02$ & $\first{0.86}\pm0.02$ & $0.80\pm0.02$ \\ 
Banknote & 1371 & 4 & 2 & $0.95\pm0.01$ & $\second{0.98}\pm0.01$ & $0.97\pm0.01$ & $\second{0.98}\pm0.01$ & $\first{0.99}\pm0.00$ \\ 
Yeast & 1484 & 8 & 10 & $\second{0.57}\pm0.02$ & $\first{0.58}\pm0.02$ & $0.54\pm0.03$ & $\first{0.58}\pm0.02$ & $0.54\pm0.02$ \\ 
Wireless localization & 2000 & 7 & 4 & $\second{0.97}\pm0.00$ & $\first{0.98}\pm0.00$ & $\first{0.98}\pm0.01$ & $\first{0.98}\pm0.00$ & $\first{0.98}\pm0.01$ \\ 
Spambase & 4600 & 57 & 2 & $\second{0.90}\pm0.01$ & $\first{0.91}\pm0.00$ & $0.86\pm0.01$ & $\first{0.91}\pm0.00$ & $0.87\pm0.01$ \\ 
Robot navigation & 5455 & 24 & 4 & $0.81\pm0.01$ & $\first{0.86}\pm0.01$ & $0.81\pm0.01$ & $\second{0.84}\pm0.01$ & $\second{0.84}\pm0.01$ \\ 
Page blocks & 5473 & 10 & 5 & $\second{0.95}\pm0.01$ & $\second{0.95}\pm0.01$ & $\first{0.96}\pm0.01$ & $\first{0.96}\pm0.01$ & $\first{0.96}\pm0.01$ \\ 
MAGIC & 19020 & 10 & 2 & $0.82\pm0.00$ & $0.82\pm0.00$ & $\first{0.84}\pm0.01$ & $\second{0.83}\pm0.00$ & $\second{0.83}\pm0.00$ \\ 
Avila & 20867 & 10 & 12 & $0.63\pm0.01$ & $0.68\pm0.01$ & $\first{0.70}\pm0.01$ & $\second{0.69}\pm0.00$ & $\first{0.70}\pm0.01$ \\ 
\toprule[0.2ex]
\end{tabular}
\end{adjustbox}
\end{table}

\section{Conclusion and future works}
\label{sec:conclusion}

In this paper, we proposed multiscale $k$-NN~(\ref{eq:msknn_definition}), that extrapolates $k$-NN estimators from $k \ge 1$ to $k=0$ via regression. 
MS-$k$-NN corresponds to finding favorable real-valued weights~(\ref{eq:optimal_weights_msknn}) for weighted $k$-NN, and it attains the convergence rate $O(n^{-(1+\alpha)\beta/(2\beta+d)})$ shown in Theorem~\ref{theo:msknn_cr}. 
It coincides with the optimal rate shown in \citet{samworth2012optimal} in the case $\alpha=1,\beta=2u \: (u \in \mathbb{N})$. 
For future work, it would be worthwhile to relax assumptions in theorems, especially the $\beta$-H{\"o}lder condition on $\mu$ and the limitation on the distance to Euclidean.
As also noted at last in \citet{samworth2012optimal} Section 4, rather larger sample sizes would be needed for receiving benefits from asymptotic theories; 
adaptation to small samples and high-dimensional settings are also appreciated.

\section*{Broader Impact}
\label{sec:broader_impact}
For improving the convergence rate of the conventional $k$-NN, we propose to consider a simple extrapolation idea; it provides an intuitive understanding of not only the optimal $k$-NN but also more general nonparametric statistics. 
By virtue of the simplicity, MS-$k$-NN is also easy to implement; the similar idea may be applied to some other statistical and machine learning methods. 

\section*{Acknowledgement}
We would like to thank Ruixing Cao and Takuma Tanaka for helpful discussions. This work was partially supported by JSPS KAKENHI grant 16H02789, 20H04148 to HS.

\newgeometry{left=25mm,right=20mm,top=20mm,bottom=20mm}

\clearpage
\onecolumn

\appendix
\begin{flushleft}
\textbf{\Large Supplementary Material:} \par
{\Large Extrapolation Towards Imaginary $0$-Nearest Neighbour \\ and Its Improved Convergence Rate}
\end{flushleft}
\hrulefill

\section{Related works}
\label{supp:related_works}

\citet{gyorfi1981rate} is the first work that proves the convergence rate $O(n^{-\gamma/(2\gamma+d)})$ for unweighted $k$-NN classifier by assuming the $\gamma$-neighbour average smoothness, and the rate is improved by \citet{chaudhuri2014rates}, by additionally imposing the $\alpha$-margin condition.

For choosing adaptive $k=k(X_*)$ with non-negative weights $w_i=1/k$, i.e., $k$ depending on the query $X_*$, 
\citet{balsubramani2019adaptive} considers the confidence interval of the $k$-NN estimator from the decision boundary, and 
\citet{cannings2017local} considers the asymptotic expansion used in \citet{samworth2012optimal} and obtains the rate of $O(n^{-4/(4+d)})$, same rate as unweighted $k$-NN up to constant factor. 
\citet{anava2016nearest} considers adaptive non-negative weights and $k=k(X_*)$ but the approach is rather heuristic.

\section{Other classifiers and their convergence rates}
\label{supp:LP}

In this section, we describe Nadaraya-Watson~(NW) classifier, Local Polynomial~(LP) classifier and their convergence rates~\citep{audibert2007learning}. 
In what follows, $K:\mathcal{X} \to \mathbb{R}$ represents a kernel function, e.g., Gaussian kernel $K(X):=\exp(-\|X\|_2^2)$, and $h>0$ represents a bandwidth. 

\begin{defi}
\emph{Nadaraya-Watson~(NW) estimator} is defined as
\begin{align*}
\hat{\eta}_{n,h}^{(\text{NW})}(X_*)
:=
\frac{\sum_{i=1}^{n}  Y_i \, K\left(
    \frac{X_i-X_*}{h}
\right)}{
\sum_{i=1}^{n}K\left(
    \frac{X_i-X_*}{h}
\right)}
\end{align*}
if the denominator is nonzero, and it is zero otherwise. 
$\hat{g}_{n,h}^{(\text{NW})}(X):=g^{(\text{plug-in})}(X;\hat{\eta}_{n,h}^{(\text{NW})})$ is called \emph{NW-classifier}. 
\end{defi}

Here, we define a loss function 
\begin{align}
\mathcal{L}_{n,h}(f,X_*)
    :=
    \scalebox{0.85}{$\displaystyle
    \sum_{i=1}^{n}
 \{Y_i-f(X_i-X_*)\}^2\,
    K\left(
        \frac{X_i-X_*}{h}
    \right)
  $}
\end{align}
for $f:\mathcal{X} \to \mathbb{R}$; Using a constant function $f(x)\equiv\theta$,
NW estimator can be regarded as a minimizer $\theta \in \mathbb{R}$ of ($\mathcal{L}_{n,h}(f,X_*)$).
NW estimator is then generalized to the local polynomial (LP) estimator when
 $Y_i$ is predicted by a polynomial function.

\begin{defi}
\label{defi:lp}
Let $\mathcal{F}_q$ denotes the set of polynomial functions $f:\mathcal{X} \to \mathbb{R}$ of degree $q\in\mathbb{N}_0$. 
Considering the function
\begin{align}
    \hat{f}_{n,h,q}^{X_*}
    :=
    \argmin_{f \in \mathcal{F}_{q}}
    \mathcal{L}_{n,h}(f,X_*), 
    \label{eq:lp_problem}
\end{align}
\emph{local polynomial~(LP) estimator} of degree $q$ is defined as $\hat{\eta}_{n,h,q}^{(\text{LP})}(X_*)
    :=
    \hat{f}_{n,h,q}^{X_*}(\bs 0)$ if $\hat{f}_{n,h,q}^{X_*}$ is the unique minimizer of $\mathcal{L}_{n,h}(f,X_*)$ and it is zero otherwise. 
The corresponding $\hat{g}^{(\text{LP})}_{n,h,q}(X):=
g^{(\text{plug-in})}(X;\hat{\eta}_{n,h,q}^{(\text{LP})})$ is called \emph{LP classifier}. 
\end{defi}

Note that LP classifier is computed via polynomial function of degree $q$; they contain $1+d+d^2+\cdots+d^q$ terms therein, and it results in high computational cost if $d,q$ are large.

\begin{prop}[\citet{audibert2007learning} Th.~3.3]
\label{prop:audibert2007_lp}
Let $\mathcal{X}$ be a compact set, and assuming that 
(i) $\eta$ satisfies $\alpha$-margin condition and is $\beta$-H{\"o}lder, and 
(ii) $\mu$ satisfies strong densitiy assumption.
Then, the convergence rate of the LP classifier with the bandwidth $h_*=h_n \asymp n^{-1/(2\beta+d)}$ is
\begin{align*}
    \mathcal{E}(\hat{g}^{(\text{LP})}_{n,h_*,\lfloor \beta \rfloor})
    =
    O(n^{-(1+\alpha)\beta/(2\beta+d)}).
\end{align*}
\end{prop}

The above Proposition~\ref{prop:audibert2007_lp} indicates that, the convergence rate for the LP classifier is faster than $O(n^{-1/2})$ for $\alpha\beta>d/2$, and the rate is even faster than $O(n^{-1})$ for $(\alpha-1)\beta>d$, though such inequalities are rarely satisfied since the dimension $d$ is large in many practical situations.

Rigorously speaking, \citet{audibert2007learning} considers the uniform bound of the excess risk over all the possible $(\eta,\mu)$, and \citet{audibert2007learning} Theorem~3.5 proves the optimality of the rate, i.e.,  $\sup_{(\eta,\mu)}\mathcal{E}(g) \ge \exists C \cdot n^{-(1+\alpha)\beta/(2\beta+d)}$ for any classifier $g$ when $\alpha \beta < d$. 
LP classifier is thus proved to be an optimal classifier in this sense.
However, the optimality is for uniform evaluation $\sup_{(\eta,\mu)}\mathcal{E}(\cdot)$, but not the non-uniform evaluation $\mathcal{E}(\cdot)$, that is considered in this paper; 
it remains unclear whether the (non-uniform) evaluation is still lower-bounded by $n^{-(1+\alpha)\beta/(2\beta+d)}$ if $\sup$ is removed. 
In particular, the uniform bound of NW classifier (i.e., LP classifier with $\lfloor \beta \rfloor=0$ ) is $O(n^{-2/(2+d)})$ for $\alpha=\beta=1$, but it is slower than the convergence rate $O(n^{-4/(4+d)})$ of NW classifier.

We last note that the LP classifier leverages the polynomial of degree $q$, that is defined in Definition~\ref{defi:lp}; it contains $1+d+d^2+\cdots+d^q$ terms, resulting in high computational cost as the dimension $d$ of feature vectors is usually not that small.

\section{A Note on Proposition~\ref{prop:chaudhuri_theorem4}}
\label{supp:note_on_chaudhuri}

Regarding the symbols, $(\alpha,\beta)$ in \citet{chaudhuri2014rates} correspond to $(\tilde{\gamma},\alpha)$ in this paper, where $\tilde{\gamma}:=\gamma/d$ is formally defined in the following. 
\citet{chaudhuri2014rates} in fact employs ``$(\alpha, L)$-smooth'' condition
\begin{align}
    |\eta(X_*)-\eta^{(\infty)}(B(X_*;r))|
    \le
    L \left(\int_{B(X_*;r)}\mu(X) \diff X \right)^{\tilde{\gamma}}, 
    \label{eq:gamma_homogeneity_1}
\end{align}
which is different from 
our definition of the $\gamma$-neighbour average smoothness, i.e., 
\begin{align}
    |\eta(X_*)-\eta^{(\infty)}(B(X_*;r))|
    \le
    L_{\gamma}r^{\gamma}.
    \label{eq:gamma_homogeneity_2}
\end{align}
However, their definition (\ref{eq:gamma_homogeneity_1}) can be obtained from our definition~(\ref{eq:gamma_homogeneity_2}), by imposing an additional assumption $\mu(X) \ge \mu_{\min}$ for all $X \in \mathcal{X}$. 
The proof is straightforward: the integrant in (\ref{eq:gamma_homogeneity_1}) is lower-bounded by
\begin{align*}
\int_{B(X_*;r)}\mu(X) \diff X
&\ge 
\mu_{\min} \frac{\pi^{d/2}}{\Gamma(1+d/2)} r^d
=: D r^d, 
\end{align*}
then
\begin{align*}
    |\eta(X_*)-\eta^{(\infty)}(B(X_*;r))|
\overset{(\ref{eq:gamma_homogeneity_2})}{\le}
    L_{\gamma} r^{\gamma}
\le 
    L_{\gamma} 
    \left(
    \frac{1}{D}
    \int_{B(X_*;r)}\mu(X) \diff X
    \right)^{\gamma/d}
=
    L\left(
    \int_{B(X_*;r)}\mu(X) \diff X
    \right)^{\tilde{\gamma}}
\end{align*}
by specifying $L:=L_{\gamma}/D^{\gamma/d},\tilde{\gamma}=\gamma/d$.
Therefore, \citet{chaudhuri2014rates} Theorem~4(b) proves Proposition~\ref{prop:chaudhuri_theorem4}, by considering the above correspondence of the symbols and the assumption.

\section{\citet{samworth2012optimal} Theorem~6}
\label{supplement:samworth}

For each $s \in (0,1/2)$, 
$\mathcal{W}_{n,s}$ denotes the set of all sequences of real-valued weight vectors $\bs w_n:=(w_1,w_2,\ldots,w_n) \in \mathbb{R}^n$ satisfying
\begin{align*}
    \sum_{i=1}^{n} w_i = 1, \quad 
    \frac{
        n^{2u/d}\sum_{i=1}^{n} \delta_i^{(\ell)} w_i
    }{
        n^{2\ell/d}\sum_{i=1}^{n} \delta_i^{(u)} w_i
    } &\le \frac{1}{\log n} \quad (\forall \ell \in [u-1]), \\
    \sum_{i=1}^{n}w_i^2 &\le n^{-s}, \\
    n^{-4u/d} (\sum_{i=1}^{n}\delta_i^{(u)}w_i)^2 &\le n^{-s}, \\
    \exists k_2 \le \lfloor n^{1-s} \rfloor \text{\: s.t. \:} 
    \frac{n^{2u/d}\sum_{i=k_2+1}^{n}|w_i|}{\sum_{i=1}^{n}\delta_i^{(u)}w_i} &\le \frac{1}{\log n} 
    \text{\: and \:}
    \sum_{i=1}^{k_2}\delta_i^{(u)}w_i \ge \beta k_2^{2u/d}, \\
    \frac{\sum_{i=k_2+1}^{n} w_i^2}{\sum_{i=1}^{n}w_i^2} &\le \frac{1}{\log n}, \\
    \frac{\sum_{i=1}^{n}|w_i|^3}{(\sum_{i=1}^{n}w_i^2)^{3/2}} &\le \frac{1}{\log n},
\end{align*}
where $\delta_i^{(\ell)}:=i^{1+2\ell/d}-(i-1)^{1+2\ell/d}$ for all $\ell \in [u-1]$.

For the rigorous proof, \citet{samworth2012optimal} considers the following assumptions. 
\begin{enumerate}[{(i)}]
    \item 
    $\mathcal{X} \subset \mathbb{R}^d$ is a compact $d$-dimensional manifold with boundary $\partial \mathcal{X}$, 
    \item 
    $\mathcal{S}:=\{x \in \mathcal{X} \mid \eta(x)=1/2\}$ is nonempty. 
    There exists an open subset $U_0 \subset \mathbb{R}^d$ that contains $\mathcal{S}$ and such that the following properties hold: 
    (1) $\eta$ is continuous on $U \setminus U_0$, where $U$ is an open set containing $\mathcal{X}$, 
    (2) restrictions of $P_0(X):=\mathbb{P}(X \mid Y=0),P_1(X):=\mathbb{P}(X \mid Y=1)$ to $U_0$ are absolutely continuous w.r.t. Lebesgue measure, with $2u$-times continuously differentiable ($C^{2u}$) Radon-Nikodym derivatives $f_0,f_1$, respectively.
    Since $f_0, f_1 \in C^{2u}$, we also have
    $\eta(x) = \mathbb{P}(Y=1) f_1(x)/(\mathbb{P}(Y=0) f_0(x) + \mathbb{P}(Y=1) f_1(x))$ is $C^{2u}$.
    
    \item 
    There exists $\rho>0$ such that $\int_{\mathbb{R}^d}\|x\|^{\rho} \diff \mathbb{P}(x)<\infty$. Moreover, for sufficiently small $r>0$, the ratio $\mathbb{P}(B(x;r))/(a_d r^d)$ is bounded away from zero, uniformly for $x \in \mathcal{X}$. 
    \item 
  $\partial \eta(x)/\partial x \neq 0$ for all $x\in \mathcal{X}$ and its restriction to $\mathcal{S}$ is also nonzero for  all $x \in \mathcal{S} \cap \partial \mathcal{X}$.
\end{enumerate}

\begin{prop}[\citet{samworth2012optimal} Theorem~6]
Assuming that (i)--(iv), it holds for each $s \in (0,1/2)$ that
\[
\mathcal{E}(\hat{g}^{(k\text{NN})}_{n,k,\bs w})
=
\underbrace{\left(
    B_1\sum_{i=1}^{n}w_i^2
    +
    B_2\left(
        \sum_{i=1}^{n}\frac{\delta_i^{(u)} w_i}{n^{2u/d}}
    \right)^2
\right)
}_{=:\gamma_n(\bs w_n)}
\{1+o(1)\}
\]
for some constants $B_1,B_2>0$, 
as $n \to \infty$, uniformly for $\bs w \in \mathcal{W}_{n,s}$, and 
$\delta_i^{(\ell)}:=i^{1+2\ell/d}-(i-1)^{1+2\ell/d},\ell \in [u-1]$. 
\end{prop}

Whereas the weights are constrained as
\begin{align}
    \sum_{i=1}^{n}w_i=1, \quad \sum_{i=1}^{n}\delta_i^{(\ell)}w_i = 0 \quad (\forall \ell \in [u-1]),
    \label{eq:supp_w_constraint}
\end{align}
and $w_i = 0$ for $i=k^*+1,\ldots,n$ with $k^* \asymp n^{2\beta/(2\beta+d)} $.
\citet{samworth2012optimal} eq.~(4.3) shows that the optimal weight should be in the form
\begin{align}
    w_i^*
    &:=
    \begin{cases}
        (a_0+a_1 \delta_1^{(i)}+\cdots + a_u \delta_i^{(u)})/k^* & (i \in [k^*]) \\
        0 & (\text{otherwise}.) \\
    \end{cases}.
    \label{eq:supp_w_optimal_samwaorth}
\end{align}
Coefficients $\bs a=(a_0,a_1,\ldots,a_u)$ are determined by solving the equations~(\ref{eq:supp_w_constraint}) and (\ref{eq:supp_w_optimal_samwaorth}) simultaneously; then the optimal weights are obtained by substituting it to (\ref{eq:supp_w_optimal_samwaorth}).

They also show the asymptotic solution of the above equations, in the case of $u=2$; the solution is
\begin{align*}
    a_1
    =
    \frac{1}{(k^{*})^{2/d}}\left\{
        \frac{(d+4)^2}{4}
        -
        \frac{2(d+4)}{d+2}a_0
    \right\},
    \quad 
    a_2
    =
    \frac{1-a_0-(k^{*})^{2/d}a_1}{(k^{*})^{4/d}}. 
\end{align*}

\section{Real-valued Weights Obtained via MS-$k$-NN}
\label{supp:optimal_weights_msknn}

Let $X_* \in \mathcal{X}$ any given query, 
and let denote $k$-NN estimator by $\varphi_{n,k}:=\hat{\eta}^{(k\text{NN})}_{n,k}(X_*)$. Considering 
\begin{align}
    \bs \varphi_{n,\bs k}
    =
    \bs \varphi_{n,\bs k}(X_*)
    &
    :=(\varphi_{n,k_1},\varphi_{n,k_2},\ldots,\varphi_{n,k_V}) \in \mathbb{R}^V, \label{eq:def_varphi} \\
    \bs R
    =
    \bs R(X_*)
    &:=
    \left(
        \begin{array}{cccc}
            r_1^2 & r_1^4 & \cdots & r_1^{2C} \\
            r_2^2 & r_2^4 & \cdots & r_2^{2C} \\
            \vdots & \vdots & \ddots & \vdots \\
            r_V^2 &, r_V^4 & \cdots & r_V^{2C} \\
        \end{array}
    \right) \in \mathbb{R}^{V \times C}, \label{eq:def_R}\\
    \bs A
    =
    \bs A(X_*)
    &:=
    (\bs 1 \: \bs R) \in \mathbb{R}^{V \times (C+1)}, \label{eq:def_A}\\
    \bs b 
    =
    \bs b(X_*)
    &:= (b_0,b_1,b_2,\ldots,b_C) \in \mathbb{R}^{C+1}, 
    \label{eq:def_b}
\end{align}
the minimization problem~(\ref{eq:msknn_estimator}) becomes
\begin{align*}
    \hat{\bs b}
    &=
    \argmin_{\bs b \in \mathbb{R}^{C+1}}
    \sum_{v=1}^{V}\left(
        \hat{\eta}^{(k\text{NN})}_{n,k}(X_*)-\sum_{c=0}^{C}b_c r_v^{2c}
    \right) 
    =
    \argmin_{\bs b \in \mathbb{R}^{C+1}}
    \|\bs \varphi_{\bs k,n} - \bs A \bs b\|_2^2
    =
    \underbrace{(\bs A^{\top}\bs A)^{-1}\bs A^{\top}}_{(\star)}\bs \varphi_{\bs k,n}.
\end{align*}
Therefore, denoting the first row of the matrix $(\star)$ by the vector $\bs z^\top=(z_1,z_2,\ldots,z_V)^\top \in \mathbb{R}^V$, 
MS-$k$-NN estimator is $\hat{\eta}^{\text{MS-$k$NN}}_{\bs k,n}(X_*)=\hat{b}_0=\bs z^{\top}\bs \varphi_{\bs k,n}$. 
To obtain the explicit form of $\bs z$, we hereinafter expand the matrix ($\star$).

Considering the inverse of block matrix
\begin{align*}
    \left(\begin{array}{cc}
        A & B \\
        C & D \\
    \end{array}\right)^{-1}
    =
    \left(\begin{array}{cc}
     (A-BD^{-1}C)^{-1} & -(A-BD^{-1}C)^{-1}BD^{-1} \\
     -D^{-1}C(A-BD^{-1}C)^{-1} & D^{-1}+D^{-1}C(A-BD^{-1}C)^{-1}BD^{-1} \\
    \end{array}\right)
\end{align*}
(see, e.g., \citet{petersen2012matrix} Section 9.1.3.), we have
\begin{align*}
    (\star)
    &=
    \left(\begin{array}{cc}
        V & \bs 1^{\top}\bs R \\
        \bs R^{\top}\bs 1 & \bs R^{\top}\bs R
    \end{array}\right)^{-1}(\bs 1 \: \bs R)^{\top} \\
    &=
    \left(\begin{array}{cc}
        \frac{1}{e} & -\frac{1}{e}\bs 1^{\top}\bs R(\bs R^{\top}\bs R)^{-1} \\
        -\frac{1}{e}(\bs R^{\top}\bs R^{-1})^{-1}\bs R^{\top}\bs 1 & 
        (\bs R^{\top}\bs R)^{-1}+
        \frac{1}{e}(\bs R^{\top}\bs R)^{-1}\bs R^{\top}\bs 1\bs 1^{\top}\bs R(\bs R^{\top}\bs R)^{-1}
    \end{array}\right)(\bs 1 \: \bs R)^{\top}, \quad \text{where}\\
    e 
    &:= 
    A-BD^{-1}C
    = 
    V-\bs 1^{\top}\bs R(\bs R^{\top}\bs R)^{-1}\bs R^{\top}\bs 1 \in \mathbb{R}.
\end{align*}
Therefore, its first column is, 
\begin{align*}
    \bs z
    &=
    \frac{1}{e}
    \left\{
        \bs I-\bs R(\bs R^{\top}\bs R)^{-1}\bs R^{\top}
    \right\}\bs 1 
    =
    \frac{1}{V-\bs 1^{\top}\bs R(\bs R^{\top}\bs R)^{-1}\bs R^{\top}\bs 1}
    \left\{
        \bs I-\bs R(\bs R^{\top}\bs R)^{-1}\bs R^{\top}
    \right\}\bs 1 
    =
    \frac{(\bs I-\mathcal{P}_{\bs R})\bs 1}{V-\bs 1^{\top}\mathcal{P}_{\bs R}\bs 1},
\end{align*}
where $\mathcal{P}_{\bs R}:=\bs R(\bs R^{\top}\bs R)^{-1}\bs R^{\top}$ represents a projection matrix;  
the equation (\ref{eq:explicit_z}) is proved.

In addition, using the vector $\bs z$, 
\begin{align*}
    \hat{\eta}^{\text{MS-$k$NN}}_{\bs k,n}(X_*)
    &=
    \bs z^{\top}
    \bs \varphi_{\bs k,n} 
    =
    \sum_{v=1}^{V}
    z_v \hat{\eta}^{(k\text{NN})}_{k,n}
    =
    \sum_{v=1}^{V} z_v \frac{1}{k_v} \sum_{i=1}^{k_v} Y_{(i)}
    =
    \sum_{i=1}^{k_V}w_i^* Y_{(i)}
    =
    \hat{\eta}^{(k\text{NN})}_{n,k_V,\bs w_*}(X_*),
\end{align*}
where 
\begin{align*}
    w_i^{*}:=\sum_{v:i \le k_v}\frac{z_v}{k_v} \in \mathbb{R}, \quad (\forall i \in [k_V]),
\end{align*}
is the real-valued weight obtained via MS-$k$-NN. Thus (\ref{eq:optimal_weights_msknn}) is proved.

\section{Proof of Theorem~\ref{theo:local_homogeneity_2}}
\label{supp:proof_prop:local_homogeneity_2}

We first prove Proposition~\ref{prop:integral_taylor} and its Corollary in the following Section~\ref{supp:homogeneity_preliminaries}; subsequently, applying the Corollary proves Theorem~\ref{theo:local_homogeneity_2}.

\subsection{Preliminaries}
\label{supp:homogeneity_preliminaries}

In this section, 
we first formally define Taylor expansion of the multivariate function in the following Definition~\ref{def:taylor}; 
Taylor expansion can approximate the function as shown in the following Proposition~\ref{prop:holder_taylor}. 
Subsequently, we consider integrals of functions over a ball, in Proposition~\ref{prop:integral_taylor} and Corollary~\ref{coro:integral_taylor}, for proving Theorem~\ref{theo:local_homogeneity_2} in Section~\ref{supp:proof_prop:local_homogeneity_2_main}.

\begin{defi}[Taylor expansion]
\label{def:taylor}
Let $d \in \mathbb{N}$ and $q \in \mathbb{N} \cup \{0\}$.
For $q$-times differentiable function $f:\mathcal{X} \to \mathbb{R}$, the Taylor polynomial of degree $q \in \mathbb{N} \cup \{0\}$ at point $X_*=(x_{*1},x_{*2},\ldots,x_{*d}) \in \mathcal{X}$ is defined as
\begin{align*}
    \mathcal{T}_{q,X_*}[f](X)
    &:=
    \sum_{s=0}^{q} \sum_{|\bs i|=s}
    \frac{(X-X_*)^{\bs i}}{\bs i!}D^{\bs i} f(X_*),
\end{align*} 
where $\bs i=(i_1,i_2,\ldots,i_d) \in (\mathbb{N} \cup \{0\})^d$ represents multi-index, $|\bs i|=i_1+i_2+\cdots+i_d$, $X^{\bs i}=x_1^{i_1}x_2^{i_2}\cdots x_d^{i_d}$, $\bs i!=i_1!i_2!\cdots i_d!$ and $D^{\bs i}=\frac{\partial^{|\bs i|}}{\partial x_{1}^{i_1}\partial x_{2}^{i_2} \cdots \partial x_{d}^{i_d}}$. 
\end{defi}

\begin{mdframed}
\begin{prop}
\label{prop:holder_taylor}
Let $d \in \mathbb{N},\beta>0$.
If $f:\mathcal{X} \to \mathbb{R}$ is $\beta$-H{\"o}lder, 
there exists a function $\varepsilon_{\beta,X_*}:\mathcal{X} \to \mathbb{R}$ such that
\begin{align*}
    f(X) = \mathcal{T}_{\lfloor \beta \rfloor,X_*}[f](X)
    +
    \varepsilon_{\beta,X_*}(X),
\end{align*}
and $|\varepsilon_{\beta,X_*}(X)|
    \le 
    L_{\beta}\|X-X_*\|_2^{\beta}
    \left( \le L_{\beta} r^{\beta},
    \: 
    \forall X \in B(X_*;r)
    \right)$, 
where $L_{\beta}$ is a constant for $\beta$-H{\"o}lder condition described in Definition~\ref{def:beta_holder}. 
\end{prop}
\end{mdframed}

\begin{proof}[Proof of Proposition~\ref{prop:holder_taylor}]
This Proposition~\ref{prop:holder_taylor} immediately follows from the definition of $\beta$-H{\"o}lder condition~(Definition~\ref{def:beta_holder}). \end{proof}

\begin{mdframed}
\begin{prop}
\label{prop:integral_taylor}
Let $d \in \mathbb{N},\beta>0$ and let $f:\mathcal{X} \to \mathbb{R}$ be a $\beta$-H{\"o}lder function. 
Then, for any query $X_* \in \mathcal{X}$, 
there exists $\tilde{\varepsilon}_{\beta} \in \mathbb{R}$ such that
\begin{align*}
    \int_{B(X_*;r)}
    f(X) \diff X
    &=
    \sum_{\substack{\bs u \in (\mathbb{N} \cup \{0\})^d \\ |\bs u| \le \lfloor \beta/2 \rfloor}}
    \frac{D^{2\bs u}f(X_*)}{(2 \bs u)!}
    \frac{g(\bs u)}{2|\bs u|+d}
    r^{2|\bs u|+d}
    +
    \tilde{\varepsilon}_{\beta},
    \quad 
    |\tilde{\varepsilon}_{\beta}|
    \le 
    L_{\beta}r^{\beta+d}
        \int_{B(\bs 0;1)} \diff x
\end{align*}
where $g(\bs u):=\frac{2\Gamma(u_1+1/2)\Gamma(u_2+1/2)\cdots\Gamma(u_d+1/2)}{\Gamma(u_1+u_2+\cdots+u_d+d/2)}$ and $\Gamma(u)$ is Gamma function. 
\end{prop}
\end{mdframed}

\begin{proof}[Proof of Propotision~\ref{prop:integral_taylor}]

Let $q:=\lfloor \beta \rfloor$. 
In this proof, we first calculate the Taylor expansion $\mathcal{T}_{q,X_*}[\eta](X)$. Then we integrate it over the ball $B(X_*;r)$, by referring to \citet{folland2001integrate},

Proposition~\ref{prop:holder_taylor} indicates that, there exists a function $\varepsilon_{\beta,X_*}(X)$ such that
\begin{align*}
    f(X)
    &=
    \mathcal{T}_{q,X_*}[f](X)
    +
    \varepsilon_{q,X_*}(X) 
    =
    \sum_{s=0}^{q} \sum_{|\bs i|=s} \frac{(X-X_*)^{\bs i}}{\bs i!}
    D^{\bs i}f(X_*)
    +
    \varepsilon_{\beta,X_*}(X)
\end{align*}
and $|\varepsilon_{\beta,X_*}(X)| \le L_{\beta} r^{\beta}$, for all $X \in B(X_*;r)$. Therefore, we have
\begin{align*}
    \int_{B(X_*;r)} f(X) \diff X
    &=
    \sum_{s=0}^{q} \sum_{|\bs i|=s}
    \frac{D^{\bs i}f(X_*)}{\bs i!}
    \underbrace{
    \int_{B(X_*;r)}
        (X-X_*)^{\bs i}
    \diff X 
    }_{(\star)}
    +
    \underbrace{
    \int_{B(X_*;r)} \varepsilon_{\beta,X_*}(X) \diff X
    }_{=:\tilde{\varepsilon}_{\beta}}.
\end{align*}
We first evaluate the term $(\star)$ in the following. 
\begin{enumerate}[{(a)}]
\item 
If at least one entry of $\bs i=(i_1,i_2,\ldots,i_d)$ is odd number, i.e., there exists $j \in [d],u \in \mathbb{N} \cup \{0\}$ such that $i_j=2u+1$, it holds that
\begin{align*}
    (\star 1)
    &=
    \int_{B(X_*;r)}(X-X_*)^{\bs i}\diff X
    =
    \int_{B(\bs 0;r)}X^{\bs i}\diff X
    =
    \int_{\substack{B(\bs 0;r') \\ r' \in [-r,r]}}
    X_{-j}^{\bs i_{-j}}
    \underbrace{
    \left\{
    \int_{-\sqrt{r^2-r'^2}}^{\sqrt{r^2-r'^2}}
    x_j^{i_j}
    \diff x_j
    \right\}
    }_{=0}
    \diff x_{-j}
    =
    0,
\end{align*}
where $X_{-j}:=(x_1,\ldots,x_{(j-1)},x_{(j+1)},\ldots,x_d) \in \mathbb{R}^{d-1}, 
\bs i_{-j}=(i_1,\ldots,i_{(j-1)},i_{(j+1)},\ldots,i_d) \in (\mathbb{N} \cup \{0\})^{d-1}$.

\item 
Therefore, in the remaining, we consider the case that all of entries in $\bs i=(i_1,i_2,\ldots,i_d)$ are even numbers, i.e., there exist $u_j \in \mathbb{N} \cup \{0\}$ such that $i_j=2u_j$ for all $j \in [d]$. 
It holds that 
\begin{align*}
    (\star 1)
    &=
    \int_{B(X_*;r)}(X-X_*)^{\bs i} \diff X
    =
    \int_{B(\bs 0;r)}X^{\bs i} \diff X \\
    &=
    \int_{0}^{r}
        \tilde{r}^{|\bs i|+d-1}
        \underbrace{
        \int_{\partial B(\bs 0;\tilde{r})}
            \tilde{X}^{\bs i} 
        \diff \sigma(\tilde{X})
        }_{=g(\bs u) \: (\because \text{ \citet{folland2001integrate}})}
    \diff \tilde{r}, \quad (\because \text{ polar coordinate}) \\
    &= 
    g(\bs u)
    \int_0^r 
        \tilde{r}^{|\bs i|+d-1}
        \diff \tilde{r} 
    =
    \frac{1}{|\bs i|+d}g(\bs u)
\end{align*}
where $\partial B(X;\tilde{r})$ denotes a surface of the ball $B(X;\tilde{r})$, $\sigma$ represents $(d-1)$-dimensional surface measure, $g(\bs u):=\frac{2\Gamma(u_1+1/2)\Gamma(u_2+1/2)\cdots\Gamma(u_d+1/2)}{\Gamma(u_1+u_2+\cdots+u_d+d/2)}$ and $\Gamma(u)$ is Gamma function. 
\end{enumerate}

Considering above (a) and (b), we have 
\begin{align*}
    \int_{B(X_*;r)}
    f(X) \diff X
    &=
    \sum_{\substack{|\bs i| \le q \\ \bs i=2\bs u,\bs u \in (\mathbb{N} \cup \{0\})^d}}
    \frac{D^{\bs i}f(X_*)}{\bs i!}
    \frac{g(\bs u)}{|\bs i|+d}
    r^{|\bs i|+d}
    +
    \tilde{\varepsilon}_{\beta} \\
    &=
    \sum_{|\bs u| \le \lfloor \beta/2 \rfloor}
    \frac{D^{2\bs u}f(X_*)}{(2 \bs u)!}
    \frac{g(\bs u)}{2|\bs u|+d}
    r^{2|\bs u|+d}
    +
    \tilde{\varepsilon}_{\beta},
\end{align*}
where $\tilde{\varepsilon}_{\beta}$ is evaluated by leveraging Proposition~\ref{prop:holder_taylor}, i.e., 
\begin{align*}
    |\tilde{\varepsilon}_{\beta}|
    &\le 
    \int_{B(X_*;r)} |\varepsilon_{\beta,X_*}(X)| \diff X
    \le 
    \sup_{X \in B(X_*;r)}|\varepsilon_{\beta,X_*}(X)|
    \int_{B(X_*;r)}\diff x
    \le 
    L_{\beta}r^{\beta}
    \int_{B(\bs 0;r)}\diff x.
\end{align*}
Therefore, the assertion is proved. 
\end{proof}

\begin{mdframed}
\begin{coro}
\label{coro:integral_taylor}
Symbols and assumptions are the same as those of Proposition~\ref{prop:integral_taylor}. Then, there exists $\tilde{\varepsilon}_{\beta} \in \mathbb{R}$ such that
\begin{align*}
    \int_{B(X_*;r)}f(X)\diff X
    =
    \frac{g(\bs 0)r^d}{d}
    f(X_*)
    +
    \sum_{c=1}^{\lfloor \beta/2 \rfloor}
    b_c r^{2c+d}
    +
    \tilde{\varepsilon}_{\beta},
    \quad 
    |\tilde{\varepsilon}_{\beta}|
    \le 
    L_{\beta}r^{\beta+d}
    \int_{B(\bs 0;1)}\diff x,
\end{align*}
where $b_c=b_c(f, X_*):=\frac{1}{2c+d}\sum_{|\bs u|=c}\frac{D^{2 \bs u}f(X_*)}{(2\bs u)!}g(\bs u)$. 
\end{coro}
\end{mdframed}

\begin{proof}[Proof of Corollary~\ref{coro:integral_taylor}]
Proposition~\ref{prop:integral_taylor} immediately proves the assertion. 
\end{proof}

\subsection{Main body of the proof}
\label{supp:proof_prop:local_homogeneity_2_main}

For the function
\begin{align}
    \eta^{(\infty)}(B(X_*;r))
    &=
    \frac{\int_{B(X_*;r)}\eta(x)\mu(x)\diff x}{\int_{B(X_*;r)}\mu(x)\diff x}, 
    \label{eq:eta_infty_mu}
\end{align}
Corollary~\ref{coro:integral_taylor} indicates that there exist 
\begin{align}
a_1&=b_1(\eta \mu,X_*),
a_2=b_2(\eta \mu,X_*),
\ldots,
a_{\lfloor \beta/2 \rfloor}=b_{\lfloor \beta/2 \rfloor}(\eta \mu,X_*) \in \mathbb{R}, \nonumber \\
b_1&=b_1(\mu,X_*),
b_2=b_2(\mu,X_*),
\ldots,
b_{\lfloor \beta/2 \rfloor}=b_{\lfloor \beta/2 \rfloor}(\mu,X_*) \in \mathbb{R}
\label{eq:definition_ab}
\end{align}
and $\tilde{\varepsilon}^{(1)}_{\beta},\tilde{\varepsilon}^{(2)}_{\beta} \in \mathbb{R}$ such that 
\begin{align}
    (\ref{eq:eta_infty_mu})
    &=
    \frac{
        \frac{g(\bs 0)r^d}{d} \eta(X_*)\mu(X_*) +\sum_{c=1}^{\lfloor \beta/2 \rfloor}a_c r^{2c+d}
        +
        \tilde{\varepsilon}_{\beta}^{(1)}
    }{
        \frac{g(\bs 0)r^d}{d} \mu(X_*) +\sum_{c=1}^{\lfloor \beta/2 \rfloor}b_c r^{2c+d} 
        +
        \tilde{\varepsilon}_{\beta}^{(2)}
    }, 
    \qquad 
    |\tilde{\varepsilon}_{\beta}^{(1)}|
    ,
    |\tilde{\varepsilon}_{\beta}^{(2)}|
    \le 
    L_{\beta} r^{\beta+d} \int_{B(\bs 0;1)} \diff x,
    \label{eq:eta_infty_mu_2}
\end{align}
since $\mu$ and $\eta \mu$ are $\beta$-H{\"o}lder.
Both the numerator and denominator are divided by $r^d$, then for sufficiently small $r>0$, the asymptotic expansion is of the form
\begin{align}
    (\ref{eq:eta_infty_mu_2})
    =
    \eta(X_*)
    +
    \sum_{c=1}^{\lfloor \beta/2 \rfloor}b^*_c(X_*) r^{2c}
    + \delta_{\beta,r}(X_*),
    \label{eq:eta_infty_mu_3}
\end{align}
where $\delta_{\beta,r}(X_*) =  O(r^{2\lfloor \beta/2 \rfloor+2}) + O(r^{\beta})$.
The two error terms are in fact combined as $\delta_{\beta,r}(X_*) = O(r^\beta)$, because $2\lfloor \beta/2 \rfloor+2 \ge \beta$.
Thus, by specifying a sufficiently small $\tilde r>0$, 
the error term is bounded as $\delta_{\beta,r}(X_*) < L_\beta^*(X_*) r^\beta$ for $r\in (0,\tilde r]$ with a continuous function $L_\beta^*(X_*) $.
For $L_\beta^* = \sup_{X\in \mathcal{S}(\mu)} L_\beta^*(X_*)< \infty$,
we have
\begin{align*}
    (\ref{eq:eta_infty_mu_3})
    &=
    \eta(X_*)
    +
    \sum_{c=1}^{\lfloor \beta/2 \rfloor} b_c^*(X_*) r^{2c}
    +
    \delta_{\beta,r}(X_*),
    \quad 
    |\delta_{\beta,r}(X_*)|
    < 
    L_{\beta}^*r^{\beta},
    \quad 
    (\forall r \in (0,\tilde{r}],
    X_* \in \mathcal{S}(\mu)).
\end{align*}

Thus proving the assertion. 
Note that, by rearranging the terms of order $r^{2+d}$, we obtain the equation 
\begin{align*}
    \frac{g(\bs 0)}{d}
    \mu(X_*)b_1^*
    +
    \eta(X_*)b_1
    =
    a_1,
\end{align*}
where $a_1:=\frac{1}{2+d}\sum_{|\bs u|=1}\frac{D^{2\bs u}(\eta(X_*) \mu(X_*))}{(2\bs u)!}g(\bs u),
b_1:=\frac{1}{2+d}\sum_{|\bs u|=1}\frac{D^{2\bs u}\mu(X_*)}{(2\bs u)!}g(\bs u)$; 
subsequently, solving the equation yields
\begin{align*}
    b_1^*
    &=
    \frac{d}{2+d}\frac{1}{\mu(X_*)}
    \sum_{|\bs u|=1} \left\{
        \frac{D^{2\bs u}(\eta(X_*) \mu(X_*))}{(2\bs u)!}
        -
        \frac{\eta(X_*) D^{2\bs u}\mu(X_*)}{(2\bs u)!}
    \right\}\frac{g(\bs u)}{g(\bs 0)} \\
    &=
    \frac{d}{2+d}\frac{1}{\mu(X_*)}
    \frac{1}{2}
    \{\Delta \eta(X_*)\mu(X_*)
    +
    \eta(X_*) \Delta \mu(X_*)\}
    \underbrace{
    \frac{
        2\Gamma(1/2)^{d-1}\Gamma(3/2)/\Gamma(1+d/2)
    }{
        2\Gamma(1/2)^d /\Gamma(d/2)
    }}_{=\frac{\Gamma(3/2)/\Gamma(1+d/2)}{\Gamma(1/2)/\Gamma(d/2)}=\frac{1/2}{d/2}=\frac{1}{d}} \\
    &=
    \frac{1}{2d+4}
    \frac{1}{\mu(X_*)}
    \{\Delta [\eta(X_*)\mu(X_*)]
    +
    \eta(X_*) \Delta \mu(X_*)\}.
\end{align*}
In general, $b_1^* \neq 0$, thus $\gamma=2$ for $\beta>2$.  For the case of $\beta=2$,  we have ${\lfloor \beta/2 \rfloor}=0$, thus 
$ (\ref{eq:eta_infty_mu_2}) = \eta(X_*) + O(r^\beta)$, meaning $\gamma=2$.
\qed

\section{Proof of Theorem~\ref{theo:msknn_cr}}
\label{supp:proof_msknn_cr}

We basically follow the proof of \citet{chaudhuri2014rates} Theorem~4(b). 
In Section~\ref{supp:msknn_definition_symbols}, we first define symbols used in this proof. 
In Section~\ref{supp:sketch_of_proof_msknn_cr}, we describe the sketch of the proof and main differences between our proof and that of  \citet{chaudhuri2014rates}~4(b). 
Section~\ref{supp:proof_main} shows the main body of the Proof, by utilizing several Lemmas listed in Section~\ref{supp:lemmas}.

\subsection{Definitions of symbols} 
\label{supp:msknn_definition_symbols}

\begin{itemize}

\item \textbf{$\bs k$ and radius $\bs r$:} 
We first specify a real-valued vector $\bs \ell=(\ell_1,\ell_2,\ldots,\ell_V)^\top \in \mathbb{R}^V$ satisfying $\ell_1=1<\ell_2<\cdots<\ell_V$. 
$k_{1,n} \asymp n^{-2\beta/(2\beta+d)}$ is assumed in (C-1), and in (C-2), 
$\{k_{v,n}\}$ are specified so that 
\begin{align*}
    k_{v,n}=\min\{k \in [n] \mid \|X_{(k)}-X_*\|_2 \ge \ell_v r_{1,n} \},
    \quad \forall v \in \{2,3,\ldots,V\}
\end{align*}
from  $r_{1,n}:=\|X_{(k_{1,n})}-X_*\|_2$.
Then, for $r_{v,n}:=\|X_{(k_{v,n})}-X_*\|_2$, $v=2,\ldots, V$,
we have $r_{v,n}/r_{1,n} \to \ell_v$.

\item \textbf{Estimators:}
Similarly to Supplement~\ref{supp:optimal_weights_msknn}, we denote the $k$-NN estimators and MS-$k$-NN estimator by 
\begin{table}[htbp]
\centering
\begin{tabular}{rl}
    (\text{Finite $k$-NN}) & 
    $\varphi_{n,k}
    =
    \varphi_{n,k}(X_*)
    :=
    \frac{1}{k}\sum_{i=1}^{k}Y_{(i;X_*)} \in \mathbb{R}$ \\
    (\text{Finite $k$-NN vector}) &
    $\bs \varphi_{n,\bs k}
    =
    \bs \varphi_{n,\bs k}(X_*)
    :=
    (\varphi_{n,k_1}(X_*),
    \varphi_{n,k_2}(X_*),
    \ldots,
    \varphi_{n,k_V}(X_*))^\top \in \mathbb{R}^V$, \\
    (\text{Finite MS-$k$-NN}) &
    $\rho_{n,\bs k}
    =
    \rho_{n,\bs k}(X_*)
    :=
    \bs z_{n,\bs k}(X_*)^{\top}\bs \varphi_{n,\bs k}(X_*) \in \mathbb{R}$, \\
\end{tabular}
\end{table}
where $\bs z_{n,\bs k}(X_*) \in \mathbb{R}^V$ denotes vectror $\bs z$ considered in Supplement~\ref{supp:optimal_weights_msknn}, i.e., 
\[
\bs z_{n,\bs k}(X_*):=\frac{(\bs I-\mathcal{P}_{\bs R_{n,\bs k}(X_*)})\bs 1}{V-\bs 1^{\top}\mathcal{P}_{\bs R_{n,\bs k}(X_*)}\bs 1}
\]
where $\mathcal{P}_{\bs R}:=\bs R(\bs R^{\top}\bs R)^{-1}\bs R^{\top}$ and the $(i,j)$-th entry of the matrix $\bs R_{n,\bs k}(X_*)$ is $r_{k_i}^{2j}=\|X_{(k_i)}-X_*\|_2^{2j}$. 
Whereas the vector $\bs z_{n,k}(X_*)$ is simply denoted by $\bs z$ in the above discussion, 
here we emphasize the dependence to the sample $\mathcal{D}_n$, parameters $\bs k=(k_1,k_2,\ldots,k_V)$ and the query $X_*$. 

We here define the asymptotic variants of the estimators by
\begin{table}[htbp]
\centering
\begin{tabular}{rl}
    (\text{Asymptotic $k$-NN}) &
    $\varphi_r^{(\infty)}
    =
    \varphi_r^{(\infty)}(X_*)
    :=
    \eta^{(\infty)}(B(X_*;r)) \in \mathbb{R}$, \\ 
    (\text{Asymptotic $k$-NN vector}) & 
    $\bs \varphi_{\bs r}^{(\infty)}
    =
    \bs \varphi_{\bs r}^{(\infty)}(X_*)
    :=
    (\varphi_{r_1}^{(\infty)}(X_*),
    \varphi_{r_2}^{(\infty)}(X_*),
    \ldots,
    \varphi_{r_V}^{(\infty)}(X_*)) \in \mathbb{R}^V$, \\
    (\text{Asymptotic MS-$k$-NN}) &
    $\rho_{\bs r}^{(\infty)}
    =
    \rho_{\bs r}^{(\infty)}(X_*)
    :=
    \bs z_{\bs r}^{\top}\bs \varphi_{\bs r}^{(\infty)}(X_*) \in \mathbb{R}$, \\
\end{tabular}
\end{table}
where $\bs r=(r_1,r_2,\ldots,r_V)$,
\[
\bs z_{\bs r}=\frac{(\bs I-\mathcal{P}_{\bs R})\bs 1}{V-\bs 1^{\top}\mathcal{P}_{\bs R}\bs 1},
\]
and the $(i,j)$-th entry of the matrix $\bs R$ is $r_i^{2j}$.

\item \textbf{Point-wise errors} for $(X_*,Y_*) \in \mathcal{X} \times \{0,1\}$ are defined as 
\begin{align*}
    R_{n,\bs k}(X_*,Y_*):=\mathbbm{1}(
    \underbrace{\rho_{n,\bs k}(X_*)}_{\text{Finite MS-$k$-NN}} \neq Y_*),
    \quad 
    R_*(X_*,Y_*):=\mathbbm{1}(g_*(X_*) \neq Y_*),
\end{align*}
where $g_*(X):=\mathbbm{1}(\eta(X) \ge 1/2)$ is the Bayes-optimal classifier equipped with $\eta(X):=\mathbb{E}(Y \mid X)$.

\item \textbf{A minimum radius} whose measure of the ball is larger than $t>0$, i.e., 
\[
    \tilde{r}_t(X)
    :=
    \inf\left\{ r>0 \: \bigg| \: 
    \int_{B(X_*;r)} \mu(X) \diff X 
    \ge t \right\}.
\]

\item \textbf{Sets for the decision boundary with margins} are defined as
\begin{align*}
    \mathcal{X}_{t,\Delta}^+
    &:=
    \left\{
        X \in \mathcal{S}(\mu) \mid \eta(X)>\frac{1}{2},
        \quad \rho^{(\infty)}_{r\bs \ell}(X) \ge \frac{1}{2}+\Delta,
        \quad
        \forall r \le \tilde{r}_t(X)
    \right\}, \\
    \mathcal{X}_{t,\Delta}^-
    &:=
    \left\{
        X \in \mathcal{S}(\mu) \mid \eta(X)<\frac{1}{2},
        \quad \rho^{(\infty)}_{r\bs \ell}(X) \le \frac{1}{2}-\Delta,
        \quad
        \forall r \le \tilde{r}_t(X)
    \right\}, \\
    \partial_{t,\Delta}
    &:=
    \mathcal{X} \setminus (\mathcal{X}_{t,\Delta}^+ \cup \mathcal{X}_{t,\Delta}^-),
\end{align*}
where $\mathcal{S}(\mu)$ is defined in (\ref{eq:smu}), and $r$ is meant for $r_1$.

\end{itemize}

\subsection{Sketch of the proof}
\label{supp:sketch_of_proof_msknn_cr}

\textbf{Sketch of the proof:} 
We mainly follow the proof of \citet{chaudhuri2014rates} Theorem~4(b), that proves the convergence rate for the unweighted $k$-NN estimator. 
Similarly to \citet{chaudhuri2014rates} Lemma~7, we first consider decomposing the difference between point-wise errors $R_{n,\bs k}(X_*,Y_*)-R_*(X_*,Y_*)$ as shown in the following Lemma~\ref{lem:decomposition}; this Lemma plays an essential role for proving Theorem~\ref{theo:local_homogeneity_2}.

Subsequently, we consider the following two steps using Lemma~\ref{lem:decomposition}--\ref{lem:term3}: 

\begin{enumerate}[{(i)}]
\item taking expectation of the decomposition w.r.t. sample $\mathcal{D}_n$ for showing point-wise excess risk, \\
(cf. \citet{chaudhuri2014rates} Lemma~20)
\item further taking expectation w.r.t. the query $(X_*,Y_*)$, and evaluate the convergence rate. \\
(cf. \citet{chaudhuri2014rates} Lemma~21)
\end{enumerate}
Then, the assertion is proved.

\textbf{Main difference} between the Proof of \citet{chaudhuri2014rates} and ours is bias evaluation. 
\citet{chaudhuri2014rates} leverages the $\gamma$-neighbour average smoothness condition 
\begin{align*}
\text{(asymptotic bias of $k$-NN)}
\qquad 
|\underbrace{\varphi_{r}^{(\infty)}(X_*)}_{\text{(asymptotic) $k$-NN}}-\eta(X_*)| \le L_{\gamma}r^{\gamma},
\end{align*}
that represents the asymptotic bias of the $k$-NN, where $\gamma$ is upper-bounded by $2$ even if highly-smooth function is employed~($\beta \gg 2$; Theorem~\ref{theo:local_homogeneity_2}). 
However, MS-$k$-NN asymptotically satisfies an inequality
\begin{align*}
    \text{(asymptotic bias of MS-$k$-NN)}
    \qquad 
    |\underbrace{\rho^{(\infty)}_{r \bs \ell}(X_*)}_{\text{(asymptotic) MS-$k$-NN}}-\eta(X_*)|
    \le 
    L_{\beta}^{**}r^{\beta}
\end{align*}
for any $\beta>0$, as formally described in Lemma~\ref{lem:bias_of_msknn}. 
By virtue of the smaller asymptotic bias, Lemma~\ref{lem:term1} proves that smaller margin is required for the decision boundary, in order to evaluate the convergence rate; it results in the faster convergence rate. 

Although the the bias evaluation is different, variance evaluation for MS-$k$-NN is consequently almost similar to the $k$-NN, as MS-$k$-NN can be regarded as a linear combination of several $k$-NN estimators, i.e., 
\[
\underbrace{\rho_{n,\bs k}(X_*)}_{\text{MS-$k$-NN estimator}} 
= 
\bs z_{n,\bs k}(X_*)^{\top} 
\underbrace{\bs \varphi_{n,\bs k}(X_*)}_{\text{$k$-NN estimators}};
\]
we adapt several Lemmas in \citet{chaudhuri2014rates} to our setting, for proving our Theorem~\ref{theo:msknn_cr}.

\subsection{Main body of the proof}
\label{supp:proof_main}

See the following Section~\ref{supp:lemmas} for Lemma \ref{lem:decomposition}--\ref{lem:term3} used in this proof. 
Throughout this proof, we assume that $X_* \in \mathcal{S}(\mu)$, as \citet{cover1967nearest} proves that $\mathbb{P}(X_* \in \mathcal{S}(\mu))=1$; the remaining $X_* \notin \mathcal{S}(\mu)$ can be ignored.

Let $n \in \mathbb{N},k_{1,n} \asymp n^{2\beta/(2\beta+d)},t_n:=2k_{1,n}/n,\Delta_o:=L_{\beta}^{***}t^{\beta/d}$ where $L_{\beta}^{***} \in (0,\infty)$ is a constant defined in Lemma~\ref{lem:term1}, and let  $\Delta(X):=|\eta(X)-1/2|$ denotes the difference between the underlying conditional expectation $\eta(X)$ from the decision boundary $1/2$.

By specifying arbitrary $i_o \in \mathbb{N}$ and $\Delta_{i_o}:=2^{i_o} \Delta_o$, we consider the following two steps (i) and (ii) for proving Theorem~\ref{theo:msknn_cr}. 
In step (i), queries are first classified into two different cases, i.e., $\Delta(X_*) \le \Delta_{i_o}$ and $\Delta(X_*) > \Delta_{i_o}$. 
Thus $i_o$ regulates the margin near the decision boundary, and it will be specified as $i_o=\max\{1,\lceil \log_2 \sqrt{\frac{2(\alpha+2)}{k_{1,n}\Delta_o^2}}\rceil\}$. 
For each case, we take expectation of the difference between point-wise errors $R_{n,\bs k}(X_*,Y_*)-R_*(X_*,Y_*)$ with respect to the sample $\mathcal{D}_n$. 
Subsequently, (ii) we further take its expectation with respect to the query $(X_*,Y_*)$; the assertion is then proved. 
Note that these steps (i) and (ii) correspond to \citet{chaudhuri2014rates} Lemma~20 and 21, respectively.

\begin{enumerate}[{(i)}]
\item 

We first consider the case $\Delta(X_*) \le \Delta_{i_o}$. Then, we have
\begin{align}
\mathbb{E}_{\mathcal{D}_n}(R_{n,\bs k}(X_*,Y_*)-R_*(X_*,Y_*)) 
&\le 
|1-2\eta(X_*)|
\mathbb{E}_{\mathcal{D}_n}\{ \mathbbm{1}(\rho_{n,\bs k}(X_*) \neq g_*(X_*)) \} \nonumber \\
&\hspace{5em} 
(\because \text{\citet{devroye1996probabilistic} Theorem~2.2}) \nonumber \\
&\le 
|1-2\eta(X_*)| \nonumber \\
&\le 
2\Delta(X_*) \nonumber\\
&\le 
2\Delta_{i_o}.
\label{eq:expectation_samples_error_1}
\end{align}

We second consider the case $\Delta(X_*)>\Delta_{i_o}$. 
Assuming that $\eta(X_*)>1/2$ without loss of generality, it holds for $r=r_{1,n}:=\|X_{(k_{1,n})}-X_*\|_2$ that
\begin{align}
    &\mathbb{E}_{\mathcal{D}_n}
    \left\{
    R_{n,\bs k}(X_*,Y_*)
    -
    R_*(X_*,Y_*)
    \right\} \nonumber \\
    &\hspace{3em}\le
    |1-2\eta(X_*)|
    \mathbb{E}_{\mathcal{D}_n}
    \left\{ \mathbbm{1}(\rho_{n,\bs k}(X_*) \neq g_*(X_*)) \right\} \nonumber \\
    &\hspace{10em} (\because \text{\citet{devroye1996probabilistic} Theorem~2.2}) \nonumber \\
    &\hspace{3em} \le 
    2\Delta(X_*)
    \mathbb{E}_{\mathcal{D}_n}
    \left\{ \mathbbm{1}(\rho_{n,\bs k}(X_*) \neq g_*(X_*)) \right\} \nonumber \\
    &\hspace{10em} (\because \text{$|1-2\eta(X_*)| \le 2\Delta(X_*)$}) \nonumber \\
    &\hspace{3em}\le
    2\Delta(X_*)\mathbb{E}_{\mathcal{D}_n}
    \big\{
    \mathbbm{1}(X_* \in \partial_{t_n,\Delta(X_*)-\Delta_{i_o}}) \nonumber \\
    &\hspace{12em}
    +
    \mathbbm{1}\left(|\rho_{n,\bs k}(X_*)-\rho^{(\infty)}_{r \bs \ell}(X_*)| \ge \frac{\Delta(X_*)-\Delta_{i_o}}{2}\right) \nonumber \\
    &\hspace{12em}
    +
    \mathbbm{1}\left(|\rho_{n,\bs k}(X_*)-\eta(X_*)| \ge \frac{\Delta(X_*)-\Delta_{i_o}}{2}\right) \nonumber \\
    &\hspace{12em}+
    \mathbbm{1}(\|X_{(k_{1,n})}-X_*\|_2 > \tilde{r}_{t_n}(X_*))
    \big\} \nonumber \\
    &\hspace{25em} (\because \text{Lemma~\ref{lem:decomposition} with $\Delta:=\Delta(X_*)-\Delta_{i_o} \in [0,1/2]$}) \nonumber \\
&\hspace{3em}\le
    2\Delta(X_*)\mathbb{E}_{\mathcal{D}_n}
    \bigg\{
    \mathbbm{1}\left(|\rho_{n,\bs k}(X_*)-\rho^{(\infty)}_{r \bs \ell}(X_*)| \ge \frac{\Delta(X_*)-\Delta_{i_o}}{2}\right) \nonumber \\
    &\hspace{12em}
    +
    \mathbbm{1}\left(|\rho_{n,\bs k}(X_*)-\eta(X_*)| \ge \frac{\Delta(X_*)-\Delta_{i_o}}{2}\right) \nonumber \\
    &\hspace{12em}
    +
    \mathbbm{1}(\|X_{(k_{1,n})}-X_*\|_2>\tilde{r}_{t_n}(X_*))
    \bigg\} 
    \hspace{5em} (\because \text{Lemma~\ref{lem:term1}, i.e., $X_* \notin \partial_{t_n,\Delta(X_*)-\Delta_{i_o}}$}) \nonumber \\
    &\hspace{3em} \le 
    2\Delta(X_*)
    \bigg\{
    \mathbb{P}_{\mathcal{D}_n}\left(|\rho_{n,\bs k}(X_*)-\rho^{(\infty)}_{r \bs \ell}(X_*)| \ge \frac{\Delta(X_*)-\Delta_{i_o}}{2}\right) \nonumber \\
    &\hspace{12em}
    +
    \mathbb{P}_{\mathcal{D}_n}\left(|\rho_{n,\bs k}(X_*)-\eta(X_*)| \ge \frac{\Delta(X_*)-\Delta_{i_o}}{2}\right) \nonumber \\
    &\hspace{12em}
    +
    \mathbb{P}_{\mathcal{D}_n}(\|X_{(k_{1,n})}-X_*\|_2>\tilde{r}_{t_n}(X_*))
    \bigg\} 
    \hspace{1em}(\because \mathbb{E}_{\mathcal{D}_n}(\mathbbm{1} (\mathcal{A}))=\mathbb{P}_{\mathcal{D}_n}(\mathcal{A})
    \text{ for any event }\mathcal{A}) \nonumber \\
    &\hspace{3em} \lesssim 
    \Delta(X_*)
    \bigg\{
        \exp\left(
            -C_1 k_{1,n}(\Delta(X_*)-\Delta_{i_o})^2
        \right)
        +
        \exp\left(
            -C_2 k_{1,n}(\Delta(X_*)-\Delta_{i_o})^2
        \right) \nonumber \\
    &\hspace{12em}
        +\exp(-L_{\bs \ell}n^{\beta/(\beta+d)}(\Delta(X_*)-\Delta_{i_o}))
        +\exp(-3k_{1,n}/2)(1+o(1)) \nonumber \\
    &\hspace{12em}
        +\exp(-n)
        +\exp\left(
            -\frac{k_{1,n}}{2}\left(
                1-\frac{k_{1,n}}{nt_n}
            \right)^2
        \right)
    \bigg\} 
    \hspace{1em}(\because \text{Lemma~\ref{lem:term2}, \ref{lem:term2_2} and \ref{lem:term3} with $\delta=k_{1,n}/nt$}) \nonumber \\
    &\hspace{3em} \lesssim
    \Delta(X_*)\exp\left(
        -C_2k_{1,n}(\Delta(X_*)-\Delta_{i_o})^2
    \right)
    +
    +\exp(-3k_{1,n}/2)(1+o(1))
    +\exp(-k_{1,n}/8)
    \label{eq:pointwise_risk} \\
    &\hspace{10em}\left(\because 
    t_n=2(k_{1,n}/n) \text{ indicates that }
    \frac{k_{1,n}}{2}\left(1-\frac{k_{1,n}}{nt_n}\right)^2 
    =
    \frac{k_{1,n}}{2}\left(
        1-\frac{1}{2}
    \right)^2 
    =
    \frac{k_{1,n}}{8}
    \right), \nonumber \\
    &\hspace{3em}\lesssim 
    \Delta(X_*)\exp\left(
        -C_2k_{1,n}(\Delta(X_*)-\Delta_{i_o})^2
    \right)+\exp(-3k_{1,n}/2)(1+o(1)), \nonumber
\end{align}
where $C_2=C_1/2=1/16V^2 L_{\bs z}^2$ is defined in Lemma~\ref{lem:term2_2}.

\item 
Excess risk of the misclassification error rate is then evaluated by
\begin{align*}
    \varepsilon(\rho_{n,\bs k})
    &=
    \mathbb{E}_{X_*,Y_*}
    \left\{
        \mathbb{E}_{\mathcal{D}_n}\left(
            R_{n,\bs k}(X_*,Y_*)-R_*(X_*,Y_*)
        \right)
    \right\} \\
    &=
    \underbrace{
    \mathbb{P}_{X_*,Y_*}
    (\Delta(X_*) \le \Delta_{i_o})
    }_{\le L_{\alpha} \Delta_{i_o}^{\alpha}\, (\because \alpha\text{-margin cond.})}
    \mathbb{E}_{X_*,Y_*}
    (
    \underbrace{
    \mathbb{E}_{\mathcal{D}_n}
    \left\{
        R_{n,\bs k}(X_*,Y_*)
        -
        R_*(X_*,Y_*)
    \right\}
    }_{\le 2\Delta_{i_o} \, (\because \text{ineq.~}(\ref{eq:expectation_samples_error_1}))} \mid \Delta(X_*) \le \Delta_{i_o})\\
    &\hspace{3em}+
    \underbrace{
    \mathbb{P}_{X_*,Y_*}
    (\Delta(X_*) > \Delta_{i_o})
    }_{\le 1}
    E_{X_*,Y_*}(
    \mathbb{E}_{\mathcal{D}_n}
        \left\{
        R_{n,\bs k}(X_*,Y_*)
        -
        R_*(X_*,Y_*)
    \right\} \mid 
    \Delta(X_*) > \Delta_{i_o}) \\
    &\lesssim
    \Delta_{i_o}^{1+\alpha}
    +
    E_{X_*,Y_*}(
    \underbrace{
    \mathbb{E}_{\mathcal{D}_n}
        \left\{
        R_{n,\bs k}(X_*,Y_*)
        -
        R_*(X_*,Y_*)
    \right\}
    }_{(\text{evaluated by ineq.~}(\ref{eq:pointwise_risk})}
    \mid 
    \Delta(X_*) > \Delta_{i_o}) \\
    &\lesssim 
    \Delta_{i_o}^{1+\alpha}
    +
    \underbrace{
    \mathbb{E}_{X_*,Y_*}\left(
        \Delta(X_*)\exp(-C_2 k_{1,n}(\Delta(X_*)-\Delta_o)^2)
        \mathbbm{1}(\Delta(X_*)>\Delta_{i_o})
    \right)}_{\lesssim \Delta_{i_o}^{1+\alpha}
    \quad (\because \text{similarly to Proof of Lemma~20 in \citet{chaudhuri2014rates}})}
    +\exp(-3k_{1,n}/2)(1+o(1)) \\
    &\lesssim
    \Delta_{i_o}^{1+\alpha}
    +\exp(-3k_{1,n}/2)(1+o(1)). 
\end{align*}
If we set $i_0=\max\{1,\lceil \log_2 \sqrt{\frac{2(\alpha+2)}{k_{1,n} \Delta_o^2}} \rceil\}$, 
\begin{align*}
    \varepsilon(\hat{\eta}^{(\text{MS-$k$NN})}_{n,\bs k})
    &=
    \varepsilon(\rho_{n,\bs k})\\
    &\lesssim 
    \Delta_o^{1+\alpha}+\exp(-3k_{1,n}/2)(1+o(1)) \\
    &\lesssim
    (2^{i_o})^{1+\alpha}\Delta_o^{1+\alpha}+\exp(-3k_{1,n}/2)(1+o(1)) \\
    &\lesssim 
    \left( 
    \max
    \left\{
    1,
    \sqrt{\frac{2(\alpha+2)}{k_{1,n} \Delta_o^2}} 
    \right\}
    \right)^{1+\alpha}
    \Delta_o^{1+\alpha}+\exp(-3k_{1,n}/2)(1+o(1)) \\
    &\lesssim 
    \max\left\{
        \Delta_o,
        \sqrt{\frac{1}{k_{1,n}}}
    \right\}^{1+\alpha}+\exp(-3k_{1,n}/2)(1+o(1)) \\
    &\lesssim 
    \max\left\{
        t_n^{\beta/d},
        \sqrt{\frac{1}{k_{1,n}}}
    \right\}^{1+\alpha}+\exp(-3k_{1,n}/2)(1+o(1)) & & (\Delta_o \asymp t_n^{\beta/d}) \\
    &\lesssim 
    \max\left\{
        \left(\frac{k_{1,n}}{n}\right)^{\beta/d},
        \sqrt{\frac{1}{k_{1,n}}}
    \right\}^{1+\alpha}
    +\exp(-3k_{1,n}/2)(1+o(1))
    & & (\because \: t_n \asymp k_{1,n}/n). 
\end{align*}
\end{enumerate}
Recalling that $k_{1,n} \asymp n^{2\beta/(2\beta+d)}$, the assertion is proved as
\begin{align*}
    \varepsilon(\hat{\eta}^{(\text{MS-$k$NN})}_{n,\bs k})
    \lesssim 
    n^{-(1+\alpha)\beta/(2\beta+d)}.
\end{align*}
\qed

\subsection{Lemmas}
\label{supp:lemmas}

We here list Lemma~\ref{lem:decomposition}--\ref{lem:term3} used in the proof for Theorem~\ref{theo:msknn_cr}. 
Roughly speaking, 
\begin{itemize}
\item Lemma~\ref{lem:decomposition} indicates the decomposition of the point-wise error. \\
(cf. \citet{chaudhuri2014rates} Lemma~7) 

\item Lemma~\ref{lem:bias_of_msknn} indicates the bias evaluation of MS-$k$-NN. 

\item Lemma~\ref{lem:concentration_z} indicates the convergence rate of $\|\bs z_{n,\bs k}(X_*)-\bs z_{r \bs \ell}\|_{\infty}$.

\item Lemma~\ref{lem:term1} adapts the first part of \citet{chaudhuri2014rates} Lemma~20 from unweighted $k$-NN to MS-$k$-NN. 

\item Lemma~\ref{lem:term2} and \ref{lem:term2_2} indicate the convergence rates related to the bias and variance evaluation of the MS-$k$-NN. \\
(cf. \citet{chaudhuri2014rates} Lemma~9) 

\item Lemma~\ref{lem:term3} indicates how fast the radius $r>0$ decreases to $0$ as $k$ increases. \\
(cf. \citet{chaudhuri2014rates} Lemma~8)
\end{itemize}

Similarly to \citet{chaudhuri2014rates}~Lemma~7, we prove the following Lemma~\ref{lem:decomposition}, that decomposes the point-wise error into four different parts.  

\begin{mdframed}
\begin{lem}
\label{lem:decomposition}
Let $g_{n,\bs k}$ be the MS-$k$-NN classifier based on sample $\mathcal{D}_n$, and let $X_* \in \mathcal{S}(\mu),t \in [0,1],\Delta \in [0,1/2]$. 
Then, it holds for $r=r_{1,n}:=\|X_{(k_{1,n})}-X_*\|_2$ that
\begin{align}
    \mathbbm{1}(g_{n,\bs k}(X_*) \neq g_*(X_*))
    &\le 
    \mathbbm{1}(X_* \in \partial_{t,\Delta}) \label{eq:term1} \\
    &\hspace{3em}+
    \mathbbm{1}(|\rho_{n,\bs k}(X_*)-
    \rho_{r \bs \ell}^{(\infty)}(X_*)|
    \ge \Delta/2) \label{eq:term2} 
    \\
    &\hspace{6em}
    \mathbbm{1}(|\rho_{n,\bs k}(X_*)-
    \eta(X_*)|
    \ge \Delta/2) \label{eq:term2_2} \\
    &\hspace{9em}+
    \mathbbm{1}(r > \tilde{r}_t(X_*)). \label{eq:term3}
\end{align}
\end{lem}
\end{mdframed}

\begin{proof}[Proof of Lemma~\ref{lem:decomposition}]
Let $\mathcal{A}$ be an event that $g_{n,\bs k}(X_*) \neq g_*(X_*)$, and let $\mathcal{B}_1,\mathcal{B}_2,\mathcal{B}_3,\mathcal{B}_4$ be events defined by the indicator functions (\ref{eq:term1})--(\ref{eq:term3}), respectively. Then, 
it suffices to prove $\mathcal{A} \Rightarrow [\mathcal{B}_1 \vee \mathcal{B}_2 \vee \mathcal{B}_3 \vee \mathcal{B}_4]$ or its contrapositive $[(\lnot \mathcal{B}_1) \wedge (\lnot \mathcal{B}_2) \wedge (\lnot \mathcal{B}_3) \wedge (\lnot \mathcal{B}_4)] \Rightarrow \lnot \mathcal{A}$, where $\lnot$ represents the negation. Here, we prove the contrapositive.

$\lnot \mathcal{B}_1$ indicates that 
$X_* \in \mathcal{X}^+_{t,\Delta}$ or $X_* \in \cup \mathcal{X}^-_{t,\Delta}$. 
\begin{itemize}
    \item 
We here consider the former case $X_* \in \mathcal{X}^+_{t,\Delta}$; then, $\lnot \mathcal{B}_4$, i.e., $r \le \tilde{r}_t(X)$, indicates that
\begin{align} 
\rho_{r \bs \ell}^{(\infty)}(X_*) \ge \frac{1}{2}+\Delta (>1/2).
\label{eq:ineq_lnot31}
\end{align}
$\lnot \mathcal{B}_2$ and $\lnot \mathcal{B}_3$ represent  
\begin{align}
    |\rho_{n,\bs k}-\rho^{(\infty)}_{r \bs \ell}(X_*)|<\Delta/2,
    \quad 
    |\rho_{n,\bs k}-\eta(X_*)|<\Delta/2,
    \label{eq:ineq_lnot2}
\end{align}
respectively; 
above inequalities (\ref{eq:ineq_lnot31}) and (\ref{eq:ineq_lnot2}) indicate 
\begin{align}
\eta(X_*)
\ge 
\rho^{(\infty)}_{r \bs \ell}(X_*)
-
|\rho_{n,\bs k}-\rho^{(\infty)}_{r \bs \ell}(X_*)|
-
|\rho_{n,\bs k}-\eta(X_*)|
>
\frac{1}{2}+\Delta-\Delta/2-\Delta/2
=
1/2.
\label{eq:eta_ineq}
\end{align}
(\ref{eq:ineq_lnot31}) and (\ref{eq:eta_ineq}) prove that both of corresponding classifiers output the same label $1$, whereupon $\lnot \mathcal{A}$. 
\item 
Similarly, for the latter case $X_* \in \mathcal{X}^-_{t,\Delta}$, both classifiers output $0$ and thus $\lnot \mathcal{A}$. 
\end{itemize}
Therefore, the assertion is proved. 
\end{proof}

\begin{mdframed}
\begin{lem}
\label{lem:bias_of_msknn}
Assuming the assumption (C-3), i.e., 
there exists $L_{\bs z} \in (0,\infty)$ such that $\|\bs z_{\bs \ell}\|_{\infty}<L_{\bs z}$. 
Then, there exist $\tilde{r},L_{\beta}^{**} \in (0,\infty)$ such that 
\begin{align*}
    |\rho^{(\infty)}_{r \bs \ell}(X_*)-\eta(X_*)|
    \le 
    L_{\beta}^{**}r^{\beta}, \:
    (\forall X_* \in \mathcal{X},
    r \in (0,\tilde{r}]).
\end{align*}
\end{lem}
\end{mdframed}

\begin{proof}[Proof of Lemma~\ref{lem:bias_of_msknn}]
Theorem~\ref{theo:local_homogeneity_2} proves 
\begin{align*}
    \underbrace{\varphi^{(\infty)}_r(X_*)}_{\text{asymptotic $k$-NN}}
    &=
    \eta(X_*)
    +
    \sum_{c=1}^{\lfloor \beta/2 \rfloor}
    b_c^* r^{2c}
    +
    \delta_r(X_*), 
    \quad
    |\delta_r(X_*)|
    \le L_{\beta}^* r^{\beta},
\end{align*}
for all $X_* \in \mathcal{S}(\mu),r \in (0,\tilde{r}]$,
for some $\tilde{r} \in (0,\infty)$; 
we have a simultaneous equation
\begin{align*}
    \bs \varphi^{(\infty)}_{r \bs \ell}
    (X_*)
    &=
    \bs A_{r \bs \ell}(X_*)\bs b_*(X_*)
    +
    \bs \delta_r(X_*),
    \quad 
    \|\bs \delta_r(X_*)\|_{\infty}
    \le 
    L_{\beta}^* r^{\beta}
    \quad 
    (\forall X_* \in \mathcal{X} ,r \in (0,\tilde{r}]),
\end{align*}
where $\bs A_{r \bs \ell}=\bs A_{r \bs \ell}(X_*)=(\bs 1 \: \bs R(X_*)) \in \mathbb{R}^{V \times (C+1)}$ is defined as same as $\bs A$ in (\ref{eq:def_A}) with the radius vector $\bs r=(r_1,r_2,\ldots,r_V)=r\bs \ell$, and the entries in 
$\bs b_*(X_*)=(\eta(X_*),b_1^*,b_2^*,\ldots,b_{\lfloor \beta/2 \rfloor}^*)$ are specified in Theorem~\ref{theo:local_homogeneity_2}.
Denoting the first entry of the vector $\bs b$ by $[\bs b]_1$,
\begin{align*}
    |\underbrace{\rho^{(\infty)}_{r \bs \ell}(X_*)}_{\text{asymptotic MS-$k$-NN}}-\eta(X_*)|
&\le 
    |
    [(\bs A^{\top}\bs A)^{-1}\bs A^{\top}\bs \varphi^{(\infty)}_{r \bs \ell}]_1
    -
    \eta(X_*)| \\
&\le 
    \bigg|
        [(\bs A^{\top}\bs A)^{-1}
        \bs A^{\top}
        \left\{
        \bs A(X_*)\bs b_*(X_*)
        +
        \bs \delta_r(X_*)
        \right\}]_1
        -
        \eta(X_*)
    \bigg| \\
&=
    |\underbrace{[\bs b_*(X_*)]_1}_{=\eta(X_*)}
    +
    \underbrace{[(\bs A^{\top}\bs A)^{-1}\bs A^{\top}\bs \delta_r(X_*)]_1
    }_{=\bs z_{r \bs \ell}^{\top} \bs \delta_r(X_*)}
    -\eta(X_*)| \\
&=
    |\bs z_{r \bs \ell}^{\top}\bs \delta_r(X_*)| \\
&\le 
    \underbrace{\|\bs z_{r \bs \ell}\|_{\infty}}_{\le L_{\bs z}}
    \underbrace{\|\bs \delta_r(X_*)\|_{\infty}}_{\le L_{\beta}^* r^{\beta}} 
    \qquad (\because \bs z_{r\bs \ell}=\bs z_{\bs \ell}, \: \forall r>0) \\
    &\le 
    L_{\bs z} L_{\beta}^* r^{\beta}.
\end{align*}
Specifying $L_{\beta}^{**}:=L_{\bs z}L_{\beta}^*$ leads to the assertion. 
\end{proof}

\begin{mdframed}
\begin{lem}
\label{lem:concentration_z}
Assuming that $X_* \in \mathcal{S}(\mu)$, (C-1) $k_{1,n} \asymp n^{2\beta/(2\beta+d)}$ and (C-2) $k_{v,n}=\min\{k \in [n] \mid \|X_{(k)}-X_*\|_2 \ge \ell_v r_{1,n}\}$ where $r=r_{1,n}:=\|X_{(k_{1,n})}-X_*\|_2$. 
Then, for sufficiently large $n \in \mathbb{N}$, there exists $L_{\bs \ell}>0$ such that
\begin{align*}
    \mathbb{P}\left(
        \|\bs z_{n,\bs k}(X_*)-\bs z_{r\bs \ell}\|_{\infty}>\Delta
    \right)
    \lesssim 
    \exp(-L_{\bs \ell} n^{\beta/(\beta+d)}\Delta)
    +
    \exp(-3k_{1,n}/2)(1+o(1)).
\end{align*}
\end{lem}
\end{mdframed}

\begin{proof}[Proof of Lemma~\ref{lem:concentration_z}]

In this proof, (i) we first evaluate the probability 
\begin{align} 
\mathbb{P}_{\mathcal{D}_n}\left(
\bigg|\frac{r_{v,n}}{r_{1,n}}-\ell_v\bigg| \ge \Delta
\right)
\label{eq:rell_probability}
\end{align}
Subsequently, (ii) evalute 
\begin{align}
    \mathbb{P}_{\mathcal{D}_n}\left(
    \|\bs z_{n,\bs k}(X_*)-\bs z_{r\bs \ell}\|_{\infty}>\Delta
    \right)
    \label{eq:z_probability}
\end{align}
by leveraging (\ref{eq:rell_probability}).

\begin{enumerate}[{(i)}]
\item 
For any positive sequence $\{b_n\}_{n \ge 1} \subset \mathbb{R}_{>0}$, 
we define $k'_{v,n}:=\min\{k \in [n] \mid \|X_{(k)}-X_*\|_2 \ge \ell_v b_n\}$. 
Although the corresponding radius $r'_{v,n}:=\|X_{(k'_{v,n})}-X_*\|_2$ is computed through the sequence $\{b_n\}$, it coincides with $r_{v,n}:=\|X_{(k_{v,n})}-X_*\|_2$ as $b_n=r_{1,n}$ will be specified later.

For any $v \in \{2,3,\ldots,V\}$, it holds that
\begin{align}
    \mathbb{P}_{\mathcal{D}_n}\left(
        \bigg| \frac{r'_{v,n}}{b_n}-\ell_v \bigg|
        \ge \Delta
    \right)
    &=
    \mathbb{P}_{\mathcal{D}_n}\left(
        r'_{v,n}-b_n \ell_v
        \ge b_n \Delta
    \right) \nonumber \\
    &=
    \mathbb{P}_{\mathcal{D}_n}\left(
        r'_{v,n} \ge b_n (\ell_v+\Delta)
    \right) \nonumber \\
    &=
    \mathbb{P}_{\mathcal{D}_n}(\forall i \in [n],X_{i} \notin 
    B(X_*;r'_{v,n}) \setminus B(X_*;b_n \ell_v)) \nonumber \\
    &\le 
    \mathbb{P}_{\mathcal{D}_n}(\forall i \in [n],X_{i} \notin 
    B(X_*;b_n (\ell_v + \Delta)) \setminus B(X_*;b_n \ell_v)). \label{eq:prob_all}
\end{align}

Considering a random variable $Z_i:=\mathbbm{1}(X_i \notin B(X_*;b_n(\ell_v+\varepsilon)) \setminus B(X_*;b_n \ell_v))$, that i.i.d. follows a Bernoulli distribution whose expectation is 
\begin{align*}
q_n
&=
1-\int_{B(X_*;b_n (\ell_v + \varepsilon)) \setminus B(X_*;b_n \ell_v)} \mu(X) \diff X \\
&\le 
1-\mu_{\min}\int_{B(X_*;b_n (\ell_v + \Delta)) \setminus B(X_*;b_n \ell_v)} \diff X \\
&\le 
1-\mu_{\min}\frac{\pi^{d/2}}{\Gamma(d/2+1)}b_n^d\{(\ell_v+\Delta)^d-\ell_v^d\} \\
&\le
1-\underbrace{\mu_{\min}\frac{\pi^{d/2}}{\Gamma(d/2+1)}d \ell_v^{d-1}}_{=:L_v}b_n^d \Delta,
\end{align*}
(\ref{eq:prob_all}) can be evaluated as
\begin{align}
    (\ref{eq:prob_all})
    =
    \mathbb{P}(Z_i=1,\forall i \in [n])
    =
    \mathbb{P}
    (Z_i=1)^n
    =
    q_n^n
    \le 
    (1-L_v b_n^d \Delta)^n.
    \label{eq:rbl_probability}
\end{align}

By leveraging (\ref{eq:rbl_probability}) and specifying $b_n=r_{1,n}$, we hereinafter evaluate (\ref{eq:rell_probability}). 
For any sequence $\{a_n\}_{n\ge 1} \subset \mathbb{R}_{>0}$, 
\begin{align*}
    \mathbb{P}_{\mathcal{D}_n}\left(
        \bigg| \frac{r_{v,n}}{r_{1,n}}-\ell_v \bigg|
        \ge \Delta
    \right)
    &=
    \int_{0}^{\infty}
    \mathbb{P}_{\mathcal{D}_n}\left(
        \bigg| \frac{r'_{v,n}}{b_n}-\ell_v \bigg|
        \ge \Delta 
    \right)
    \mathbb{P}_{\mathcal{D}_n}\left(
        r_{1,n}=b_n
    \right) \diff b_n \\
&\le 
    \left\{
    \int_0^{a_n}
    +
    \int_{a_n}^{\infty}
    \right\}
    \mathbb{P}_{\mathcal{D}_n}\left(
        \bigg| \frac{r'_{v,n}}{b_n}-\ell_v \bigg|
        \ge \Delta
    \right)
    \mathbb{P}_{\mathcal{D}_n}\left(
        r_{1,n}=b_n
    \right) \diff b_n \\
&\le 
\underbrace{
\mathbb{P}_{\mathcal{D}_n}\left(
    \bigg| \frac{r_{v,n}}{b_n}-\ell_v \bigg|
    \ge \Delta
    \mid 
    b_n>a_n
\right)}_{\le (1-L_1 a_n^d \Delta)^n}
\mathbb{P}_{\mathcal{D}_n}(r_{1,n} > a_n) \\
&\hspace{5em}+
\mathbb{P}_{\mathcal{D}_n}\left(
    \bigg| \frac{r_{v,n}}{b_n}-\ell_v \bigg|
    \ge \Delta
    \mid 
    b_n \le a_n
\right)
\mathbb{P}_{\mathcal{D}_n}(r_{1,n} \le a_n) \\
&\lesssim 
\underbrace{(1-L_v a_n^d \Delta)^n}_{(\star 1)}
+
\underbrace{\mathbb{P}_{\mathcal{D}_n}(r_{1,n} \le a_n)}_{(\star 2)}. 
\end{align*}

By specifying $a_n:=n^{-1/(\beta+d)}$, the terms $(\star 1),(\star 2)$ are evaluated as follows. 
\begin{enumerate}[{(a)}]
\item 
Regarding $(\star 1)$, it holds that
\begin{align*}
    (\star 1)
    &=
    (1-L_1 n^{-d/(\beta+d)}\Delta)^n
    \le 
    \exp\left(
        -n^{\beta/(\beta+d)}L_1 \Delta
    \right),
\end{align*}
as $(1-1/a)^b \le ((1-1/a)^{a})^{b/a} \le \exp(-1)^{b/a} =\exp(-b/a)$ for all $a,b>0$. 

\item 
Here we evaluate the second term ($\star 2$): considering a random variable $Z_i:=\mathbbm{1}(X_i \in B(X_*;a_n))$ that i.i.d. follows a Bernoulli distribution whose expectation is
\begin{align*}
    q_n':=\int_{B(X_*;a_n)} \mu(X) \diff X
    \le 
    \mu_{\min} \int_{B(X_*;a_n)} \diff X
    \le 
    \mu_{\min} \frac{\pi^{d/2}}{\Gamma(d/2+1)}a_n^d
    \lesssim 
    n^{-d/(\beta+d)},
\end{align*}
we have an inequality 
\begin{align*}
    \mathbb{P}(r_{1,n} \le a_n)
    &=
    \mathbb{P}\left(\sum_{i=1}^{n}Z_i \ge k_{1,n}\right) 
    =
    \mathbb{P}\left(\sum_{i=1}^{n}Z_i \ge nq_n'
    +\lambda\right) \quad (\text{where }\lambda:=k_{1,n}-nq_n')\\
    &\le
    \exp\left(
        -\frac{\lambda^2}{2(nq_n+\lambda/3)}
    \right) \\
    &\le
    \exp\left(
        -\frac{(k_{1,n}-nq_n')^2}{2(nq_n'+(k_{1,n}-nq_n')/3)}
    \right) \\
    &\lesssim 
    \exp(-3k_{1,n}/2)(1+o(1)) \hspace{3em} (\because nq_n' = o(k_{1,n}))
\end{align*}
by referring to a Chernoff bound~\citep[][Theorem~2.4]{chung2006complex} with $\mathbb{E}_{\mathcal{D}_n}(\sum_{i=1}^{n} Z_i)=nq_n'$.
\end{enumerate}

Therefore, above (a) and (b) yield
\begin{align} 
\mathbb{P}_{\mathcal{D}_n}\left(
\bigg|\frac{r_{v,n}}{r_{1,n}}-\ell_v\bigg| \ge \Delta
\right) 
\lesssim 
\exp\left(
    -n^{\beta/(\beta+d)}L_1 \Delta
\right)
+
\exp\left(
    -3k_{1,n}/2
\right)(1+o(1)). 
\end{align}

\item We second evaluate $(\ref{eq:z_probability})$. As it holds that
\begin{align*}
    (\bs 1^{\top}&(\bs I-\mathcal{P}_{\bs R_{n,\bs k}})\bs 1)(\bs 1^{\top}(\bs I-\mathcal{P}_{\bs R})\bs 1)
    \|\bs z_{n,\bs k}-\bs z_{r_{1,n} \bs \ell}\| \\
    &=
    (\bs 1^{\top}(\bs I-\mathcal{P}_{\bs R_{n,\bs k}})\bs 1)(\bs 1^{\top}(\bs I-\mathcal{P}_{\bs R})\bs 1)
    \bigg\|
    \frac{
        \bs I-\mathcal{P}_{\bs R_{n,\bs k}}
    }{
        \bs 1^{\top}(\bs I-\mathcal{P}_{\bs R_{n,\bs k}})\bs 1
    }
    -
    \frac{
        \bs I-\mathcal{P}_{\bs R}
    }{
        \bs 1^{\top}(\bs I-\mathcal{P}_{\bs R})\bs 1
    }\bigg\|_{\infty} \\
    &=
    \|
        (\bs 1^{\top}(\bs I-\mathcal{P}_{\bs R_{n,\bs k}})\bs 1)(\bs I-\mathcal{P}_{\bs R})
        -
        (\bs 1^{\top}(\bs I-\mathcal{P}_{\bs R})\bs 1)(\bs I-\mathcal{P}_{\bs R_{n,\bs k}})
    \|_{\infty} \\
    &\le 
    \|
        (\bs 1^{\top}(\bs I-\mathcal{P}_{\bs R_{n,\bs k}})\bs 1)(\bs I-\mathcal{P}_{\bs R})
        -
        (\bs 1^{\top}(\bs I-\mathcal{P}_{\bs R})\bs 1)(\bs I-\mathcal{P}_{\bs R})
    \|_{\infty} \\
    &\hspace{8em}+
    \|
        (\bs 1^{\top}(\bs I-\mathcal{P}_{\bs R})\bs 1)(\bs I-\mathcal{P}_{\bs R})
        -
        (\bs 1^{\top}(\bs I-\mathcal{P}_{\bs R})\bs 1)(\bs I-\mathcal{P}_{\bs R_{n,\bs k}})
    \|_{\infty} \\
    &\le 
    |\bs 1^{\top}(\mathcal{P}_{\bs R_{n,\bs k}}-\mathcal{P}_{\bs R})\bs 1|
    \|(\bs I-\mathcal{P}_{\bs R})\|_{\infty}
    +
    |\bs 1^{\top}(\bs I-\mathcal{P}_{\bs R})\bs 1|
    \|\mathcal{P}_{\bs R_{n,\bs k}}-\mathcal{P}_{\bs R}\|_{\infty}. \\
    &\le 
    \|\bs 1\|_{\infty}^2
    \|\bs I-\mathcal{P}_{\bs R}\|_{\infty}
    \|\mathcal{P}_{\bs R}-\mathcal{P}_{\bs R_{n,\bs k}}\|_{\infty},
\end{align*}

there exist constants $L^{(1)},L^{(2)}>0$ such that 
\begin{align*}
    \|\bs z_{n,\bs k}(X_*)-\bs z_{r\bs \ell}\| 
\le 
    L^{(1)} \|\mathcal{P}_{\bs R_n}-\mathcal{P}_{\bs R}\|_{\infty} 
\le 
    L^{(2)} \|\bs r_n/r_{1,n}-\bs \ell\|_{\infty},
\end{align*}
where $\bs r_n=(r_{1,n},r_{2,n},\ldots,r_{V,n}) \in \mathbb{R}^V$. 
\end{enumerate}

Consequently, above (i) and (ii) yield  
\begin{align*}
    \mathbb{P}\left(
        \|\bs z_{n,\bs k}(X_*)-\bs z_{r\bs \ell}\|_{\infty}>\Delta
    \right)
    &\le 
    \mathbb{P}(L^{(2)}\|\bs r_n/r_{1,n}\|_{\infty}>\Delta) \\
    &\lesssim 
    \exp(-L_{\bs \ell} n^{\beta/(\beta+d)}\Delta)
    +
    \exp(-3k_n/2)(1+o(1)).
\end{align*}
for some constant $L_{\bs \ell}>0$. 
\end{proof}

\begin{mdframed}
\begin{lem}[Evaluation for (\ref{eq:term1})]
\label{lem:term1}
Let
\begin{itemize}
\item $X_* \in \mathcal{S}(\mu),
\beta>0,
t \in [0,1],i_o \in \mathbb{N}$,
\item $L_{\beta}^{***}:=L_{\beta}^{**}\tilde{L}^{-\beta/d}$, where $\tilde{L}:=
(\sup_{X \in \mathcal{X}}\mu(X))\frac{\pi^{d/2}}{\Gamma(d/2+1)}$ and $L_{\beta}^{**}$ is defined in Lemma~\ref{lem:bias_of_msknn}. 
\item $\Delta_o:=L_{\beta}^{***}t^{\beta/d},\Delta_{i_o}:=2^{i_o} \Delta_o$.
\end{itemize}
If $\Delta(X_*) > \Delta_{i_o}$, 
it holds that $X_* \notin \partial_{t,\Delta(X_*)-\Delta_{i_o}}$.
\end{lem}
\end{mdframed}

\begin{proof}[Proof of Lemma~\ref{lem:term1}]
For any $r \in (0,\tilde{r}_t(X_*)]$, 
\begin{align} 
t 
\le 
\int_{B(X_*;r)} \mu(X) \diff X 
\le 
\left(\sup_{X \in \mathcal{X}}\mu(X)\right)
\int_{B(X_*;r)} \diff X
=
\left(\sup_{X \in \mathcal{X}}\mu(X)\right)
\frac{\pi^{d/2}}{\Gamma(d/2+1)}r^d
=
\tilde{L}r^d.
\label{eq:t_uppder_bound}
\end{align}
Assuming that $\eta(X_*)>1/2$ without loss of generality, we have
\begin{align*}
\rho^{(\infty)}_{r\bs \ell}(X_*) 
&\ge 
\eta(X_*)-L_{\beta}^{**} r^{\beta} 
&
&(\because \text{Lemma~\ref{lem:bias_of_msknn}})
\\
&\ge 
\eta(X_*)-L_{\beta}^{**}(\tilde{L}^{-1/d}t^{1/d})^{\beta} 
&
&(\because \text{ineq.~(\ref{eq:t_uppder_bound})})\\
&=
\eta(X_*)-(L_{\beta}^{**}\tilde{L}^{-\beta/d})t^{\beta/d} \\
&=
\eta(X_*)-\Delta_o 
& 
&(\because \Delta_o=(L_{\beta}^{**}\tilde{L}^{-\beta/d})t^{\beta/d}) \\
&=
\eta(X_*)-2^{-i_o}\Delta_{i_o} 
&
&(\because \Delta_{i_o}=2^{i_o}\Delta_o) \\
&=
\frac{1}{2}+(\Delta(X_*)-2^{-i_o}\Delta_{i_o}) 
&
&(\because \Delta(X_*)=|\eta(X_*)-1/2|,\eta(X_*)>1/2) \\
&\ge 
\frac{1}{2}+(\Delta(X_*)-\Delta_{i_0})
& &
(\because \Delta_{i_o} \ge 2^{-i_o}\Delta_{i_o})
\end{align*}
for any $r \in (0,\tilde{r}_t(X_*)]$; 
it means that $X_* \in \mathcal{X}^+_{t,\Delta(X_*)-\Delta_{i_o}}$, whereupon $X_* \notin \partial_{t,\Delta(X_*)-\Delta_{i_o}}$. 
Similar holds for the case $\eta(X_*)<1/2$. 
Thus we have proved $X_* \notin \partial_{t,\Delta(X_*)-\Delta_{i_o}}$. 
\end{proof}

\begin{mdframed}
\begin{lem}[Evaluation for (\ref{eq:term2})]
\label{lem:term2}
Let $X_* \in \mathcal{X},\Delta \in [0,1/2]$ and $r_{1,n}:=\|X_{(k_{1,n})}-X_*\|_2$. Then, it holds for $C_1=1/8V^2 L_{\bs z}^2$ that
\begin{align*}
    \mathbb{P}_{\mathcal{D}_n}\bigg(
        |\rho_{n,\bs k}(X_*)
        -
        &\rho^{(\infty)}_{r_{1,n} \bs \ell}(X_*)|
        \ge 
        \Delta/2
    \bigg) \\
    &\lesssim 
    \exp(-C_1 k_{1,n} \Delta^2) 
    +
    \exp(-L_{\bs \ell}n^{\beta/(\beta+d)}\Delta)
    +
    \exp(-3k_{1,n}/2)(1+o(1)).
\end{align*}
\end{lem}
\end{mdframed}

\begin{proof}[Proof of Lemma~\ref{lem:term2}]
By simply decomposing the terms, we have
\begin{align}
|\rho_{n,\bs k}(X_*)-\rho^{(\infty)}_{r_{1,n} \bs \ell}(X_*)| 
&=
|\underbrace{\bs z_{n,\bs k}(X_*)^{\top}\bs \varphi_{n,\bs k}(X_*)}_{=\rho_{n,\bs k}(X_*)}
-
\bs z_{r_{1,n} \bs \ell}^{\top}\bs \varphi_{n,\bs k}(X_*)| \label{eq:first_term} \\
&\hspace{7em}
+
|\bs z_{r_{1,n} \bs \ell}^{\top}\bs \varphi_{n,\bs k}(X_*)
-
\bs z_{r_{1,n} \bs \ell}^{\top}\bs \varphi^{(\infty)}_{r_{1,n} \bs \ell}(X_*)| \label{eq:second_term} 
\end{align}
where the terms (\ref{eq:first_term}), (\ref{eq:second_term}) are evaluated as follows.
\begin{enumerate}[{(i)}]
    \item Regarding the first term~(\ref{eq:first_term}), 
    \begin{align*}
        (\ref{eq:first_term})
        &=
        |\{\bs z_{n,\bs k}(X_*)-\bs z_{r_{1,n} \bs \ell}\}
        \bs \varphi_{n,\bs k}(X_*)|
        \le
        \|\bs z_{n,\bs k}(X_*)-\bs z_{r_{1,n} \bs \ell}\|_{\infty}
        \underbrace{\|\bs \varphi_{n,\bs k}(X_*)\|_{\infty}}_{\le 1}.
    \end{align*}
    Therefore, Lemma~\ref{lem:concentration_z} leads to 
    \begin{align*}
        \mathbb{P}((\ref{eq:first_term}) \ge \Delta/4)
        &\le 
        \mathbb{P}(\|\bs z_{n,\bs k}(X_*)-\bs z_{r_{1,n} \bs \ell}\|_{\infty} \ge \Delta/4) \\
        &\lesssim 
        \exp(-L_{\bs \ell}n^{\beta/(\beta+d)}\Delta)
        +
        \exp(-3k_n/2)(1+o(1)),
    \end{align*}
    for some constant $L_{\bs \ell}>0$. 
    
    \item Regarding the second term~(\ref{eq:second_term}),
    \begin{align*}
        (\ref{eq:second_term})
        &=
        |\bs z_{r\bs \ell}^{\top}
        \{\bs \varphi_{n,\bs k}(X_*)
        -
        \bs \varphi^{(\infty)}_{r \bs \ell}(X_*)\} |
        \le 
        \underbrace{\|\bs z_{r\bs \ell}\|_{\infty}}_{\le L_{\bs z}}
        \sum_{v=1}^{V}
        |\varphi_{n,k_v}(X_*)
        -
        \varphi^{(\infty)}_{r h_v(X_*)}(X_*)|,
    \end{align*}
    and \citet{chaudhuri2014rates} Lemma~9 proves that 
    \begin{align*} 
    \mathbb{P}\left(
        |\varphi_{n,k_v}(X_*)-\varphi^{(\infty)}_{r_v}(X_*)| \ge \Delta/4VL_{\bs z}
    \right) \lesssim \exp(-2k_v(\Delta/4VL_{\bs z})^2).
    \end{align*}
    Therefore, we have 
    \begin{align*}
        \mathbb{P}((\ref{eq:second_term}) \ge \Delta/4)
        &\lesssim 
        \mathbb{P}\left(
            L_{\bs z}
            \sum_{v=1}^{V}
            |\varphi_{n,k_v}(X_*)
            - \varphi^{(\infty)}_{r_v}(X_*)| \ge \Delta/4
        \right) \\
        &\le
        \sum_{v=1}^{V}
        \mathbb{P}\left(
            |\varphi_{n,k_v}(X_*)
            -\varphi^{(\infty)}_{r_v}(X_*)|
            \ge \Delta/4VL_{\bs z}
        \right) \\
        &\lesssim 
        \exp(-2k_1\Delta^2/(4VL_{\bs z})^2)
        =
        \exp(-k_1 C_1 \Delta^2), 
    \end{align*}
    with $C_1:=1/8V^2 L_{\bs z}^2$. 
\end{enumerate}

Considering above evaluations, we have 
\begin{align*}
    \mathbb{P}(|\rho_{n,\bs k}(X_*)-\eta(X_*)| \ge \Delta/2)
    &\le 
    \mathbb{P}((\ref{eq:first_term}) \ge \Delta/4)
    +
    \mathbb{P}((\ref{eq:second_term}) \ge \Delta/4) \\
    &\lesssim 
    \exp(-C_1 k_1 \Delta^2) 
    +
    \exp(-L_{\bs \ell}n^{\beta/(\beta+d)}\Delta)
    +
    \exp(-3k_n/2)(1+o(1)).
\end{align*}

The assertion is proved.
\end{proof}

\begin{mdframed}
\begin{lem}[Evaluation for (\ref{eq:term2_2})]
\label{lem:term2_2}
Let $X_* \in \mathcal{X}$ and $\Delta \in [0,1/2]$. Then, it holds for $C_2=1/(2VL_{\bs z})^2(=C_1/2)$ that
\begin{align*}
    \mathbb{P}_{\mathcal{D}_n}\left(
        |\rho_{n,\bs k}(X_*)
        -
        \eta(X_*)|
        \ge 
        \Delta/2
    \right)
    \lesssim 
    \exp\left(
        - C_2 k_{1,n} \Delta^2
    \right)
    +
    \exp(-n).
\end{align*}
\end{lem}
\end{mdframed}

\begin{proof}[Proof of Lemma~\ref{lem:term2_2}]
By simply decomposing the terms, we have
\begin{align}
|\rho_{n,\bs k}(X_*)-\eta(X_*)| 
&\le
\underbrace{|\rho_{n,\bs k}(X_*)-\rho^{(\infty)}_{r_{1,n} \bs \ell}(X_*)|}_{(\star 1)}
+
\underbrace{|\rho^{(\infty)}_{r_{1,n} \bs \ell}(X_*)-\eta(X_*)|}_{(\star 2)}. 
\label{eq:2_2_1}
\end{align}

\begin{itemize}
    \item Regarding the first term $(\star 1)$, applying Lemma~\ref{lem:term2} immediately leads to 
    \begin{align*}
        \mathbb{P}((\star 1) \ge \Delta/4)
        \lesssim 
        \exp(-C_2 k_{1,n}\Delta^2),
    \end{align*}
    where $C_2:=C_1/2=1/16V^2 L_{\bs z}^2$. 
    
    \item Here, we consider the second term $(\star 2)$. 
    As Lemma~\ref{lem:bias_of_msknn} shows that $|\rho^{(\infty)}_{r \bs \ell}(X_*)-\eta(X_*)| \le L_{\beta}^{**}r_{1,n}^{\beta}$, we have
\begin{align}
    \mathbb{P}(
    |\rho^{(\infty)}_{r_{1,n} \bs \ell}(X_*)-\eta(X_*)| \ge 
    \Delta/2)
    \le 
    \mathbb{P}(L_{\beta}^{**} r_{1,n}^{\beta} \ge \Delta/2)
    =
    \mathbb{P}(r_{1,n} \ge (\Delta/2L_{\beta}^{**})^{1/\beta})
    \label{eq:r1n_L}
\end{align}
(\ref{eq:r1n_L}) represents the probability that 
less than $k_{1,n}$ out of $n$ feature vectors lie in a region $B(X_*;\Delta_*)$ with $\Delta_*:=(\Delta/2L_{\beta}^{**})^{1/\beta}$; considering a random variable $Z_i :=\bs 1 (X_i \in B(X_*;\Delta_*))$, that i.i.d. follows a Bernoulli distribution whose expectation is $q_*:=\int_{B(X_*;\Delta_*)}\mu(X) \diff X>0$, 
\begin{align*}
    (\ref{eq:r1n_L})
    &=
    \mathbb{P}\left(\bar{Z}_n < \frac{k_{1,n}}{n}\right) 
    \le 
    \mathbb{P}\left(|\bar{Z}_n-q_*| \ge q_*-\frac{k_{1,n}}{n}\right) 
    \le 
    2\exp\left(-2n\left(q_*-\frac{k_{1,n}}{n}\right)^2\right)
\end{align*}
by H{\"o}effding's inequality. 
As $\frac{k_{1,n}}{n} \asymp n^{-d/(2\beta+d)} \le q_*/2$ for sufficiently large $n$, we have $(\ref{eq:r1n_L}) \lesssim \exp(-n)$.  
\end{itemize}

Considering above ($\star 1$) and ($\star 2$)
\begin{align*}
    \mathbb{P}(|
    \rho_{n,\bs k}(X_*)
    -\eta(X_*)| \ge \Delta)
&\le 
    \mathbb{P}((\star 1) \ge \Delta/2)
    +
    \mathbb{P}((\star 2) \ge \Delta/2) 
\lesssim 
    \exp(-C_2 k_{1,n} \Delta^2)
    +
    \exp(-n)
\end{align*}
for some $C_2>0$; the assertion is then proved.
\end{proof}

\begin{mdframed}
\begin{lem}[Evaluation for (\ref{eq:term3})]
\label{lem:term3}
Let $X_* \in \mathcal{X},t,\delta \in [0,1]$ and $k \in [(1-\delta)nt]$. 
Then, 
\begin{align*}
    \mathbb{P}_{\mathcal{D}_n}(\|X_{(k)}-X_*\|_2 > \tilde{r}_t(X_*))
    \lesssim 
    \exp(-k \delta^2/2). 
\end{align*}
\end{lem}
\end{mdframed}

\begin{proof}[Proof of Lemma~\ref{lem:term3}]
The assertion is obtained by \citet{chaudhuri2014rates}~Lemma~8. 
\end{proof}

\end{document}